\documentclass{article}

      \usepackage[preprint, nonatbib]{neurips_2019}

\usepackage{neurips_2019}

\usepackage[utf8]{inputenc} % allow utf-8 input
\usepackage[T1]{fontenc}    % use 8-bit T1 fonts
\usepackage{url}            % simple URL typesetting
\usepackage{booktabs}       % professional-quality tables
\usepackage{amsfonts}       % blackboard math symbols
\usepackage{nicefrac}       % compact symbols for 1/2, etc.
\usepackage{microtype}      % microtypography

\usepackage{amsmath,amsfonts,amsthm,amssymb,bm,bbm,verbatim,dsfont,mathtools}
\usepackage{color,graphicx,appendix}
\usepackage{algorithm}
\usepackage{algorithmic}
\usepackage{tabularx, booktabs}
\usepackage{subfigure}
\usepackage{enumitem} 
\usepackage{subfigure}

\newtheorem{lemma}{Lemma}
\newtheorem{theorem}{Theorem}

\usepackage{authblk}

\usepackage{multirow}

\title{Dynamic Cell Structure via Recursive-Recurrent Neural Networks}

\author[1]{Xin Qian}
\author[2]{Matthew Kennedy}
\author[3]{Diego Klabjan}
\affil[1,3]{Department of Industrial Engineering and Management Sciences, Northwestern University}
\affil[2]{Weinberg College of Arts and Sciences, Northwestern University}

\begin{document}

\maketitle

\begin{abstract}
	    \noindent In a recurrent setting, conventional approaches to neural architecture search find and fix a general model for all data samples and time steps. We propose a novel algorithm that can dynamically search for the structure of cells in a recurrent neural network model. Based on a combination of recurrent and recursive neural networks, our algorithm is able to construct customized cell structures for each data sample and time step, allowing for a more efficient architecture search than existing models. Experiments on three common datasets show that the algorithm discovers high-performance cell architectures and achieves better prediction accuracy compared to the GRU structure for language modelling and sentiment analysis. 
\end{abstract}

\section{Introduction}
	    First proposed by Hopfield \cite{hopfield1982neural}, recurrent Neural Network (RNN) models excel at machine learning tasks that involve sequential data such as natural language processing. Researchers soon noted that a major obstacle of RNN models is in backpropagation when computing gradients. Since RNNs are trained by backpropagation through time, when the recurrent structure is unfolded into a huge feed-forward network with many layers, gradients tend to grow or vanish exponentially in the same way as in very deep feed-forward neural networks \cite{pascanu2012understanding}. Many extensions of RNN models, such as Long Short-Term Memory (LSTM) \cite{hochreiter1997long} and Gated Recurrent Units (GRU) \cite{chung2014empirical}, are proposed to address this problem. These models achieve state-of-the-art results in many machine learning tasks like language modeling \cite{de2015survey} and speech recognition \cite{li2015constructing, sak2014long}.
	    
	    However, the cell structure of these hand-crafted RNN models, like LSTM and GRU, is fixed across all time steps and data samples. It is also a time-consuming and tedious effort to find a suitable cell structure through trial and error \cite{miller1995neural}. Lastly, there is no universal answer to which cell structure to use when facing different types of data and a different problem at hand. Therefore, a more flexible model that can automatically determine the cell structures based on a finite set of trainable parameters is needed to deal with more and more complicated and diversified data sources and problems. 
	    
	    There is another line of research about Recursive Neural Network (RecNN) models \cite{socher2011parsing}. A RecNN model is defined over recursive tree structures -- each node of the tree corresponds to a vector computed from its child nodes, and the information passes from the leaf nodes and internal nodes to the root node in a bottom-up manner. The model produces a structured prediction such as a tree by applying the same set of trainable parameters recursively. Derivatives of errors are computed with back-propagation over the tree structures \cite{goller1996learning}. RecNN has shown great success in learning tree structures of certain natural language processing tasks \cite{socher2013recursive} because the structures it dynamically produces are customized for each data sample. 
	    
	    We consider how to replace the cell structure in RNN models to be time-variant and sample-dependent. We note that the equations governing a cell can be represented as a computational tree where each non-leaf node corresponds to a vector that is computed from the vectors on its two child nodes. The initial multiset of vectors is composed of the current feature vector at time $t$ and all vectors produced by the previous cell (hidden state representation). If we augment this multiset with constant vectors, such as the zero vector, we can then express mathematical equations behind a cell as a tree on this multiset. RecNN is an appropriate model to capture such a tree by means of a finite set of trainable parameters. In summary, our proposed model is using RecNN in each time step as a replacement for a fixed set of equations. In this way we obtain an architecture with cells depending on time and on each individual sample. In addition to this flexibility, the approach does not require hand-crafting of cells.  
	    
        Our model shows great results on a series of language modeling and sentiment prediction tasks. In the experiments we show that RRNN is able to design sample-dependent tree structures on the Wikipedia dataset and achieves 5.5\% improvement in Bits per Character (BPC) compared to GRU. The performance on the datasets also show the advantage of dynamically designing cell structures for each sample. 
        
	    The major contribution of this paper is a novel architecture that dynamically searches for the structure of cells in an RNN. Our model, called a Recursive-Recurrent Neural Network (RRNN), recursively designs the cell structure with the help of a scoring function and allows us to build different cell structures under a fixed set of parameters. The proposed model can generate the cell structure of some traditional RNN models, like GRU and LSTM which we establish theoretically. Most importantly, the output tree structure of hidden cells in RRNN are customized based on each data sample, and therefore they are time-variant and data-dependent. Besides, we define a new tree distance metric that can measure the difference between the tree with vectors on each of its nodes. We also exhibit and prove the sufficient and necessary conditions for avoiding the gradient exploding and vanishing problem that usually appears in recurrent neural network models. While such results are known for RNNs, they have not yet been established for RecNNs. Furthermore, our result applies to RRNNs which are a combination of RNN and RecNN.
	    
	    The rest of the manuscript is structured as follows. In Section \ref{litreview} we review the literature while in Section \ref{sec:RRNNmodel} we present the RRNN model, including an algorithm to construct trees, the design of the loss function, and other extensions. Section \ref{sec:properties} presents some properties of the RRNN model. In Section \ref{sec:experiments} we introduce the data sets and discuss all experimental results. We defer the proofs of the theorems and other technical details to Appendix.
\section{Literature Review}\label{litreview}
    	A recurrently connected structure in RNN can improve the performance of a model by its ability to infer sequential dependencies \cite{lipton2015critical}. Despite their success, vanilla RNN models are still limited by the algorithms employed due to the problems of exploding or vanishing gradients that may appear in the training phase \cite{279181}. LSTM \cite{hochreiter1997long} is one of the most popular ways to address this problem. Many variants are then proposed to improve the performance of LSTM \cite{graves2005framewise, kalchbrenner2015grid}. RNN models often work well if a hand-crafted cell structure is well-designed, which requires time and expertise, and it leads to a fragile setting that works only on a particular problem or, worse, on a single dataset. This is clearly less general and less flexible than the method proposed in this paper where the cells are algorithmically designed.
    
    	Recursion is the division of a problem into subproblems of the same type and the application of an algorithm to each subproblem. It can help with augmenting neural architectures and improving the generalization ability of a model \cite{cai2017making}. RecNN greedily searches hierarchical tree structures and achieves state-of-the-art performances on tasks like semantic analysis in natural language processing and image segmentation \cite{socher2011parsing, socher2013recursive}. 
    		
    	To provide better flexibility and robustness, automatically searching a neural network architecture is thus a logical next step. Neural Architecture Search (NAS), a subfield of AutoML, is a method which algorithmically finds an architecture; it has significant overlap with hyper-parameter optimization and meta-learning \cite{elsken2018neural}. A simple approach to NAS is to build a layer-chained neural network where layers are differentiated by their choices of operations (pooling, convolution, etc.), activation functions (ReLU, Sigmoid, etc.), width, etc. \cite{chollet2017xception, yu2015multi, baker2016designing}. Despite its impressive empirical performance, NAS is computationally expensive and time consuming \cite{zoph2018learning}.
    	
    	Various methods of producing novel cell structures for RNNs have been recently proposed. \cite{zoph2016neural} introduce a reinforcement learning approach that utilizes policy gradient to search for convolutional and recurrent neural architectures. However, the reinforcement learning approach is computationally expensive in the sense that obtaining an architecture with state-of-the-art performance on CIFAR-10 and ImageNet requires 1,800 GPU days \cite{zoph2018learning}. \cite{pham2018efficient} accelerate the search process by sharing parameters among potential architectures. \cite{schrimpf2017flexible} introduce a more flexible algorithm that searches for novel RNNs of arbitrary depth and width. \cite{liu2018darts} relax the discrete architecture space by continuous probability vectors and utilize a gradient based optimization method to derive an optimal architecture. All these methods are extremely computationally demanding and they yield a fixed network architecture for all times and samples. Some exceptions are in \cite{graves2016adaptive} and \cite{zhang2018layer} where the proposed models automatically adjust the number of layers of the LSTM model based on time and sample but the cell structures are static. Our RRNN model further extends this property such that the predicted cell structures are time-variant and sample-dependent.
	
	\section{Recursive-Recurrent Neural Network Model} \label{sec:RRNNmodel}
		Generally, RNNs consist of two parts which are a hidden cell (recurrent cell) and an output layer. A single sample input of an RNN is a sequence of vectors $ \left\{x_t \in \mathbb{R}^p: t=1,2,\ldots, T\right\} $, labeled by time step $ t $. Given a hidden state $h_{t-1}$, the $ t $-th recurrent cell defines the next hidden state $ h_t $ by $h_t = f(x_t, h_{t-1})$. The output layer is usually a simple network that takes $x_t$ and $h_t$ as input and returns $q_t = g(x_t, h_t; \Gamma )$ as output. These two equations are applied for $t = 1, 2, \ldots, T$.
		
        Function $f$ defined above is time-invariant and thus remains the same in all time steps and for all samples. To address this shortcoming, we propose a new model that can dynamically design the recurrent cell structure (i.e. generate different functions $f$) with respect to the argument vectors. This is inspired by the idea of RecNNs, thus we call it the Recursive-Recurrent Neural Network model. A dynamic architecture has two advantages: (1) no need to hand-craft a cell, and (2) it automatically adjusts based on timestep and sample.
		
		A simple RecNN model starts with a set of input nodes $ \left\{ p_1, \ldots, p_n \right\} $ with corresponding embedding vectors $\left\{ c_1, \ldots, c_n \right\}$. Two nodes are merged into a parent node using a pair of weight matrices $L$ and $R$, a bias vector $b$, and an activation function $\sigma$ that provides non-linearity. For two nodes $p_i$ and $p_j$, their parent, denoted by $p_{i,j}$, is also a node with the embedding vector calculated by  $c_{i,j} = \sigma\left(Lc_i + Rc_j + b \right)$. In each iteration, we compute the scores $s_{i,j} = W^{\mathrm{score}} c_{i,j}$ for all pairs of nodes $ (p_i, p_j) $ and select the pair of nodes $(p_{i_1}, p_{j_1})$ with the highest score. We next merge nodes $p_{i_1}$ and $p_{j_1}$ into the parent node $p_{i_1, j_1}$ and remove the two child nodes $p_{i_1}$ and $p_{j_1}$ from further consideration.  This procedure repeats until all nodes are merged and only one parent node $ p_{\mathrm{out}} $ remains. The set of parameters and activation function $\{L, R, b, \sigma, W^{\mathrm{score}}\} $ are shared across the whole network. The RecNN model returns $ p_{out} $ and the binary tree rooted at $ p_{\mathrm{out}} $ as the model output. 

		The RRNN model replaces the fixed hidden cell of RNN by a recursive tree, dynamically determined by an algorithm similar to RecNN. Note that, even with the fixed set of parameters and activation function $ \left\{L, R, b, \sigma, W^{\mathrm{score}}\right\} $, the RecNN model can dynamically produce different tree structures based on input nodes (vectors). Therefore, in RRNN, the recurrent cell is different across all time steps and data points. We further discuss the RRNN model in the following sections.

	\subsection{Recursive-Recurrent Neural Network Model Framework} \label{sec:model}
		We start with an example of how to represent the hidden cell structure of GRU to be a binary tree with computational information on it. In the following we assume $X$ is a given sample, where $ X = (x_1, \ldots, x_T), x_i \in \mathbb{R}^p $ is a sequence of input vectors. Recall that the GRU equations are:	\begin{align}
		r_t &= \sigma \left(W_r x_{t} + W_r' h_{t-1} + b_r\right), \label{eq:GRUeq1}\\
		z_t &= \sigma \left(W_z x_{t} + W_z' h_{t-1} + b_z\right), \\
		\tilde{h}_t &= \tanh \left(W_h x_{t} + W_h' \left(r_t \odot h_{t-1}\right) + b_h\right), \\
		h_t &= z_t \odot h_{t-1} + (1-z_t)\odot \tilde{h}_t, \label{eq:GRUeq4}
		\end{align}
		where $ \{W_r, W_r', W_z, W_z', W_h, W_h'\}$ and $  \{b_r, b_z, b_h  \}$ are parameter matrices and bias vectors of GRU, respectively. Equations (\ref{eq:GRUeq1}) -- (\ref{eq:GRUeq4}) jointly define the function $ f $ of the $ t $-th hidden cell of GRU. As shown in Figure \ref{fig:GRUtree1}, the above equations can also be regarded as a binary tree where each node of the tree corresponds to a 3-tuple (binary operator, activation function, bias vector), and each edge is associated with a trainable matrix or identity matrix. 
		
    	\begin{figure}[!h]
		    \centering
		    \includegraphics[width=.75\textwidth]{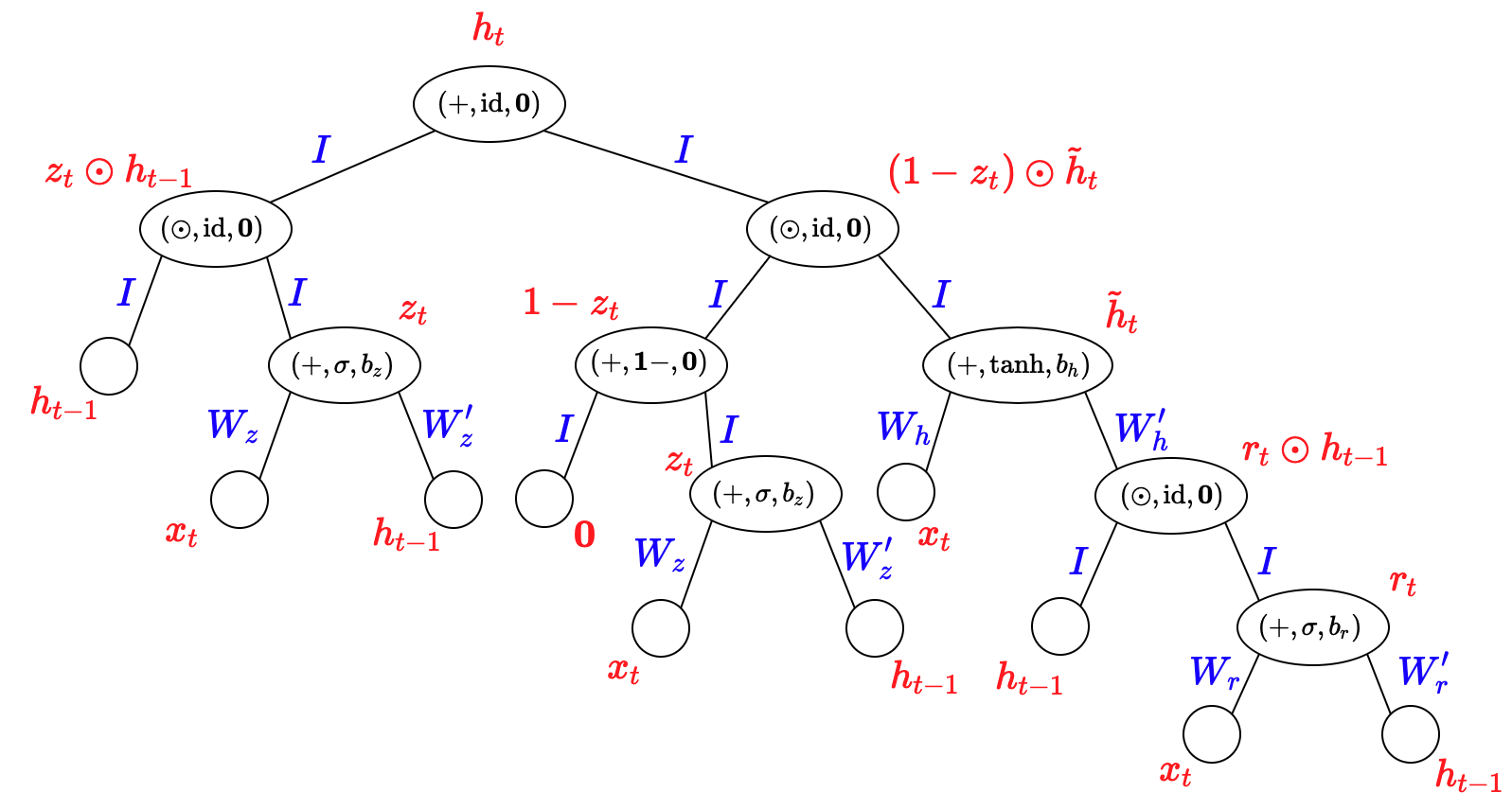}
		    \caption{Representation of GRU equations as a binary computational tree. Here red labels are the vectors correspond to each node, blue labels are the weight matrices corresponding to each edge from a child node to its parent, and the 3-tuple in each node is the binary operation, activation function, and the bias vector that are used to calculate the vector of this node. $\mathbf{0}$ stands for the zero vector, $\mathbf{1}-$ the one-minus activation and $\mathrm{id}$ the identity activation (mapping).}
			\label{fig:GRUtree1}
		\end{figure}
		
		The tree structure can be achieved in the scheme of RecNN by giving a multiset of initial nodes $ \calN_0 = \left\{x_t, x_t, x_t, x_t, h_{t-1}, h_{t-1}, h_{t-1}, h_{t-1}, h_{t-1}, \mathbf{0}\right\} $. Assuming an appropriate scoring function, in the first iteration, we can find that the parent node that combines $ x_t $ and $ h_{t-1} $ with parameters and operations $ (W_r, W_r', b_r, +, \sigma) $ has the highest score, thus we merge two nodes $ x_t $ and $ h_{t-1} $ together to achieve node $ q_t $. In the second iteration, we find that the parent node $ z_t = \sigma \left(W_z x_{t} + W_z' h_{t-1} + b_z\right) $ has the highest score among all potential parent nodes, thus we again take two nodes $ x_t $ and $ h_{t-1} $ from the node set and merge them to be $ z_t $. After $ 9 $ iterations, we end up with one node $ h_t $ and this is exactly the output of the $t$-th hidden cell of GRU. We can prove that, with an appropriate choice of the scoring function, a RecNN can find the tree in Figure \ref{fig:GRUtree1} and thus it can produce GRU. The statement is given in Section \ref{sec:grurrnn} and the proof is in the Appendix.

        Next, we present the RRNN model. The RRNN model has the same recurrent structure as RNN and an algorithm for building the hidden cell. We start with the case where only one hidden state needs to be transferred between two consecutive hidden cells (GRU falls in this case, but in LSTM we have two states, i.e. the hidden state $h_t$ and the memory state). We extend the model to be compatible with multiple hidden states in Appendix \ref{sec:extension}. 
		
		For each hidden cell of RRNN, we build up a binary tree from a multiset of initial nodes with corresponding vectors on each node. A set of trainable parameter matrices and bias vectors, activation functions, binary operations, and a scoring function is given prior to the construction of the tree. We denote $ \calL = \{L_1, \ldots, L_{n_l}\}$ and $ \calR = \{R_1, \ldots, R_{n_l}\} $ with $ L_i, R_i \in \mathbb{R}^{p\times p} $ as the set of trainable weight matrices, $\calB = \{b_1, \ldots, b_{n_l}\}$ with $ b_i \in \mathbb{R}^p $ as the set of trainable bias vectors, $\mathcal{U} = \{u_1, u_2, \ldots, u_{n_u}\}$ and $\mathcal{O} = \{o_1, o_2, \ldots, o_{n_o}\}$ as the set of available activation (unary) functions and binary operations, respectively. Besides, a scoring function $ \alpha(\cdot; \Theta) $, depending on a set of trainable parameters $ \Theta $, is given. A multiset of initial nodes $\calN_0 = \{c_1, c_2, \ldots, c_N\}$ with $c_i \in \mathbb{R}^p$ is also given as input of the hidden cell. We do not distinguish between a node and its corresponding vector in the following discussion, however we note that nodes in the tree are unique while the corresponding vectors form a multiset. The tree shown in Figure \ref{fig:GRUtree1} is the equivalent RRNN representation for the GRU cell.
		
		Formally, the RRNN hidden cell can be understood as a function $ f: \calN_0 \rightarrow (\calT, \mathbb{R}^p)$, where $ \calT $ is the set of all possible (binary) computational trees such that each node of the tree corresponds to a vector and a 3-tuple $ (u, o, b), u\in \calU, o \in \calO, b \in \calB$, and each directed edge from each one of the two child nodes to its parent node is associated with a weight matrix. 
		
		Function $f$ can be recursively defined as \begin{equation}\label{eq:f}
		    f = f_N \circ f_{N-1} \circ \cdots \circ f_2\circ f_1, 
		\end{equation} where $f_k: \calN_{k-1} \mapsto \calN_k, k=1,2,\ldots, N-1 $ and $f_N: \calN_{N-1} \rightarrow (\calT, \mathbb{R}^p)$. For $k=1,\ldots,N-1$, function $f_k $ maps multiset $\calN_{k-1}$ to multiset $\calN_k$ by the following three steps: (i) $ C_k = \left\{ c: c =u(o(L c_i, R c_j)+b), c_i, c_j \in \calN_{k-1}, i < j, L \in \calL, R \in \calR, b \in \calB, u \in \calU, o \in \calO \right\}$, (ii) $ c_k^* = \arg \max\limits_{c} \left\{\alpha(c; \Theta): c \in C_k \right\} $, (iii) $ \calN_k = \{c_k^*\} \cup \calN_{k-1} \setminus \{c_i^*, c_j^*\} $, where $ c_i^*, c_j^*$ are the two child nodes combined to get $ c_k^* $.

	    Note that $\calN_{N-1} = f_{N-1} \circ \cdots \circ f_1(\calN_0) $ contains only one node, i.e. $\calN_{N-1} = \left\{c_{N-1}^*\right\}$. Function $f_N$ then takes $c_{N-1}^*$ and returns the tree rooted at $c_{N-1}^*$ (we can discover it by unfolding the collapsing decisions and tracing each parent node down to its child nodes until all initial nodes appear) and the corresponding vector $c_{N-1}^* \in \mathbb{R}^p $ as the output. We point out that by definition the produced binary tree is full, i.e. each node has exactly 2 or 0 child nodes.  
	    
	    We next specify the recursive relationship of our cells. To this end, let multiset $\calN_0^t=\calN_0^t(x_t,h_{t-1})$ consist of several copies of $x_t$, several copies of $h_{t-1}$ and other constant vectors such as the vector of all zeros or all ones or unit vectors. The numbers of each of them can vary by $t$. The transition equations and cell output are as follows:
	    \begin{equation*}
    	    (T_t^{\pred},h_t)=f(\calN_0^t(x_t,h_{t-1})), \quad q_t=g(x_t,h_t;\Gamma).
	    \end{equation*} 
    	It remains to specify the loss function. The generic function is as follows with further details provided in Section \ref{sec:loss}. 
    	We assume that a sample consists of $ (X, Y)$ where $ Y = (y_{1}, \ldots, y_{T}) $ is a sequence of ground truth labels. We also assume that we are given a ground truth binary tree $T_t^{\target}$ which is specified as in Figure \ref{fig:GRUtree1} but without the trainable matrices and bias vectors. The target tree usually does not depend on $t$. Ideally this target tree should not be specified but we leave this as future research work. 
    	
    	One further complication is the fact that the ground truth tree does not have a unique representation. Indeed, since the leaf nodes corresponding to $\calN_0$ are unordered, there are several isomorphisms of a given tree that yield the same underlying ground truth transition function, i.e. mathematically equivalent expressions. To this end, let ${\Iso}(T_t^{\target})$ be the set of all isomorphic trees to $T_t^{\target}$. Note that we do not need to consider the isomorphisms when leaf nodes are ordered, as is the case in \cite{socher2011parsing}.
    	
    	The set of all trainable parameters in RRNN is denoted by $\Phi=\{\calL, \calR, \calB, \Theta, \Gamma \}$. The loss function is specified by \begin{equation}
    	    \begin{aligned} \label{eq:loss}
    	    L(\Phi)=\mathbb{E}_{(X,Y)}\Bigg[\sum_{t=1}^T \bigg\{ \lambda_1 l(y_t,q_t)& +\lambda_2 \min_{\bar{T}\in {\Iso}(T_t^{\target})} \mathrm{TD}(\bar{T},T_t^{\pred}) +\lambda_3 \sum_{k=0}^{N-1} m(\calN_k^t)\bigg\}\Bigg] \\ & + \lambda_4 \sum_{\phi \in \Phi} \|\phi \|^2 \;,
    	\end{aligned}
    	\end{equation}
    	where function $l$ is the standard loss function, $TD$ measures the difference of two trees, and $m$ is the margin function. These two are described in detail in Section \ref{sec:loss}. The \emph{minimum} operation over isomorphic target trees can also be replaced by expectation. 

\subsection{Cell Tree Construction} 

        Several changes of constructing the cell tree are made for practical concerns. Functions $ f_k, k = 1,\ldots, N-1 $ can be regarded as $ N-1 $ iterations of merging two nodes (vectors). Multisets $\calN_k$ can have multiple copies but in practice we keep a single copy that is reused. The new set $\calN_{k-1}$ consists of three fixed sets of vectors, namely $\calS_t^{\text{data}} = \{x_t\}$ as the set of vectors from the data samples, $\calS_t^{\text{prev}}$ as the set of vectors from the previous hidden cell, and $\calS_t^{\text{aux}}$ as the set of auxiliary vectors such as the zero vector, etc., together with the set $\calP_{k-1}$ as the set of generated parent nodes. The model takes the new set $\calN_{k-1} = \calS_t^{\text{data}} \cup \calS_t^{\text{prev}} \cup \calS_t^{\text{aux}} \cup \calP_{k-1}$ as the set of all potential choices of child nodes to build the $k$-th parent node $c_k^*$. Then we set $\calP_k = \calP_{k-1} \cup \{c_k^*\}$ and step to the $(k+1)$-th iteration. We further need a hyper-parameter $\bar{N}$ corresponding to the number of iterations of the tree construction steps. The practical algorithm for constructing the computational tree for the $ t $-th hidden cell of RRNN is exhibited in Algorithm \ref{algo:2}. It is worth mentioning that the number of iterations $\bar{N}$ in Algorithm \ref{algo:2} might be different from the number of nodes $N$ in the predicted tree. A vector might be chosen several times to serve as a child node in Algorithm \ref{algo:2}. In this case, the number of nodes $N$ in the predicted tree is larger than the number of iterations $\bar{N}$.
		
		\begin{algorithm}[ht]
			\caption{Construction of computational tree for $ t $-th hidden cell}
			\label{algo:2}
			\begin{algorithmic}[1]
				\STATE Input: $ \calL, \calR, \calB, \calU, \calO, \alpha, \mathcal{S}_t^{data}, \mathcal{S}_t^{prev},  \mathcal{S}_t^{aux}, \bar{N} $
				\STATE Output: 
				\STATE \quad $ h_t $: The hidden state of $ t $-th hidden cell
				\STATE \quad $ T_t^{\pred} $: Binary computational tree corresponds to the $ t $-th hidden cell
				\STATE
				\STATE $ \mathcal{P}_0 \leftarrow \emptyset$
				\FOR{$ k = 1$ to $ \bar{N} $}
				\STATE $ V_t^k \leftarrow \emptyset $
				\FOR{$ r = 1$ to $ |\calL| $, all $ o \in \mathcal{O} $, and all $ u \in \mathcal{U} $}
				\FOR{$c_i, c_j \in \mathcal{S}_t^{data} \cup \mathcal{S}_t^{prev} \cup \calS_t^{\text{aux}} \cup \mathcal{P}_{k-1} $, $ i < j $}
				\STATE $ V_t^k \leftarrow V_t^k \cup \{u(o(L_r c_i, R_r c_j)+b_r)\} $
				\ENDFOR
				\ENDFOR
				\STATE $c_k^{*} \gets \arg \max \left\{ \alpha(v; \Theta): v \in V_t^k, v \not\in \calP_{k-1} \right\}$
				\STATE $\calP_{k} \gets \calP_{k-1} \cup \{c_k^{*}\}$
				\ENDFOR
				%\STATE
				%\STATE // Now we have $ \calP_{\bar{N}} = \left\{ c_1^{*}, \ldots, c_{\bar{N}}^{*} \right\}$
				\STATE $ h_t \gets c_{\bar{N}}^{*} $, $ T_t^{\pred} \gets $ the tree rooted at $ c_{\bar{N}}^{*} $
				\STATE Return $ h_t $ and $ T_t^{\pred} $
			\end{algorithmic}
		\end{algorithm}

	\subsection{Loss Function} \label{sec:loss}

        We discuss the definition of the tree distance (TD) and the scoring margin $m$ in this section.
		\paragraph{Score Margin}
    		To give scoring more partitioning power, we incentivize it to leave a significant margin between the score of the highest-scoring vector and the second-highest vector for each node. Recall the definition of $C_k$ and $c_k^*$ from Section \ref{sec:model}. We further define $ c_k^{**} $ to be the vector with the second highest score among the vectors in $C_k$. In Algorithm \ref{algo:2} the analogous to $C_k$ is $V_t^k$. The scoring margin function is thereby defined as $ m(\calN_k) = -\frac{1}{M} \min \{ M, \alpha(c_k^*; \Theta) - \alpha(c_k^{**}; \Theta)\}, $ where $ M $ is a hyper-parameter. Intuitively, the margin function incentivizes scoring to increase the gap between the scores of the highest and second-highest vectors to at least $M$. We divide by $M$ so that the overall scale of this loss term is not affected by the choice of $M$.
		
		\paragraph{Tree Distance}		
		For convenience, in the discussion of this part, we use $T^{\mathrm{pd}}$ and $ T^{\mathrm{tgt}} $ to denote the predicted tree and the target (ground truth) tree, respectively. For any binary tree $\bar{T}$, we use $\mathrm{Int}(\bar{T})$ to denote all internal (non-leaf) nodes of $\bar{T}$.  We use $ \calI(\bar{T}) $ to denote the labeling of $ \mathrm{Int}(\bar{T}) $ such that the root node of $\bar{T}$ has index 1, and if a node has index $i$, then its left and right child nodes have index $2i$ and $2i+1$, respectively. For a node $n \in \mathrm{Int}(\bar{T})$, we use $ \Subtree(\bar{T}, n) $ to denote the subtree of $ \bar{T} $ rooted at node $ n $. In addition, we use $n_{i, \bar{T}}$ and $ v_{i, \bar{T}} $ to denote the node and the corresponding vector with index $ i $ in tree $\bar{T}$, respectively. 
		
		Given two binary trees $ T_1 $ and $ T_2 $, we define \begin{equation*}
		\mathrm{VD}(T_1, T_2) = \sum_{i \in \calI(T_1)\cap \calI(T_2)} \left\|v_{i, T_1} - v_{i, T_2}\right\|^2 + \sum_{i \in \calI(T_1)\setminus \calI(T_2)} \left\|v_{i, T_1}\right\|^2 + \sum_{i \in \calI(T_2)\setminus \calI(T_1)} \left\|v_{i, T_2}\right\|^2 
		\end{equation*}
		to be the vector differences (VD) of these two trees. The tree distance between $T^{\mathrm{pd}}$ and $T^{\mathrm{tgt}}$ is the sum over all minimum VD values between a sub-tree of $T^{\mathrm{pd}}$ and all sub-trees of $T^{\mathrm{tgt}}$: \begin{equation*}
		    \mathrm{TD}(T^{\mathrm{pd}}, T^{\mathrm{tgt}}) = \sum_{n_1 \in V(T^{\mathrm{pd}})} \min_{n_2 \in V(T^{\mathrm{tgt}})} \Big\{  \mathrm{VD}\big(\Subtree(T^{\mathrm{pd}}, n_1), \Subtree(T^{\mathrm{tgt}}, n_2)\big) \Big\}.
		\end{equation*}
        This expression matches each subtree in $T^{\mathrm{pd}}$ with the closet subtree in $T^{\mathrm{tgt}}$ with respect to VD, and therefore the TD measures the difference of vectors on all of the nodes of the two trees.

    \section{Properties of RRNN and Gradient Control} \label{sec:properties}
    In this section, we state some properties of the RRNN model and show how to avoid gradient exploding and vanishing during training of RRNN. We give theorems in this section and defer the proofs to the appendix.
    
    \subsection{Expressibility of RRNN}\label{sec:grurrnn}
        We argue that if we carefully choose sets $\calL, \calR, \calB, \calU, \calO, \mathcal{S}_t^{data}, \mathcal{S}_t^{prev},  \mathcal{S}_t^{aux}$, the quantity $N$, and the scoring function $\alpha$, then Algorithm \ref{algo:2} can replicate the GRU and LSTM equations. We give the formal statements in this section and defer the choice of the sets and the proof to Appendix \ref{appendix:expressibility}.
        
    	\begin{theorem} \label{thm:GRU}
    		There exists a scoring function $ \alpha $ such that Algorithm \ref{algo:2} generates GRU equations (\ref{eq:GRUeq1}) -- (\ref{eq:GRUeq4}) with an appropriate choice of $\calL, \calR, \calB, \calU, \calO, \mathcal{S}_t^{data}, \mathcal{S}_t^{prev}$, and $  \mathcal{S}_t^{aux}$.
    	\end{theorem}
    
        \begin{theorem} \label{thm:LSTM}
        	There exists a scoring function $ \alpha $ such that Algorithm \ref{algo:2} (applied twice) generates the LSTM equations with an appropriate choice of $\calL, \calR, \calB, \calU, \calO, \mathcal{S}_t^{data}, \mathcal{S}_t^{prev}$, and $  \mathcal{S}_t^{aux}$.
    	\end{theorem}
    
    \subsection{Controlling Gradient}
    	As introduced in \cite{279181}, the exploding gradient problem refers to the large increase in the norm of the gradient during training. This is due to the fact that the gradient of long-term dependencies grows exponentially quicker than for short-term dependencies. The vanishing gradient problem, on the other hand, refers to the behavior that the gradients of long-term dependencies go to zero exponentially. \cite{pascanu2012understanding} introduce a sufficient condition of vanishing gradient and a necessary condition of exploding gradient for a simple RNN. In this section, we extend their results to a more general case -- we provide these two conditions for our RRNN model. We note that our result as a special case applies to RecNN where such conditions have not yet been established. 
    	
    	We consider the case where only one hidden state $h_t$ is returned by the $t$-th hidden cell of the RRNN model. The loss function (\ref{eq:loss}) can be written as $L(\Phi) = \sum_{t=1}^T \mathcal{E}_t $ where each $\calE_t$ is a function of all parameters in $\Phi$. For $1\le t \le T$, the gradient of $\calE_t$ with respect to $\phi \in \Phi$ comes from $t$ cells, namely $\frac{\partial \calE_t}{\partial \phi} = \sum_{t^\prime=1}^t \frac{\partial \calE_t}{\partial h_t} \frac{\partial h_t}{\partial h_{t^\prime}} \frac{\partial^+ h_{t^\prime}}{\partial \phi}$, where $\frac{\partial^+ h_{t^\prime}}{\partial \phi}$ refers to the direct gradient of $h_{t^\prime}$ with respect to $\phi$ directly appearing within the $t^\prime$-th hidden cell. If $\phi$ is a matrix, then we mean $\frac{\partial^+ h_{t^\prime}}{\partial \phi} = \frac{\partial^+ h_{t^\prime}}{\partial \mathrm{vec}(\phi)}$, where $\mathrm{vec}(\phi)$ is an appropriate matrix vectorization. The exploding (vanishing) gradient problem is defined by $\left\| \frac{\partial \calE_t}{\partial h_t} \frac{\partial h_t}{\partial h_{t^\prime}} \frac{\partial^+ h_{t^\prime}}{\partial \phi} \right\|$ going to $+\infty$ ($0$) exponentially fast as $t$ goes to $+\infty$ and $t^\prime$ is fixed as a constant. For simplicity, we consider the case where $t = T$ and $t^\prime = 1$. %We formally state the theorem in the following and defer the proofs and discussions to Appendix \ref{appendix:gradientcontrol}. 
    	
    	We state a simplified version of the theorems here and defer the full version to Appendix \ref{appendix:gradientcontrol}. We argue that most of the time these conditions are met in practice and we elaborate them one by one in Appendix \ref{appendix:discussionvanishing}.
    	
    	\begin{theorem}[Sufficient condition of gradient vanishing] \label{thm:vanishing} Under certain conditions given in Theorem \ref{thm:vanishingapp}, we have $\left\| \frac{\partial \calE_T}{\partial h_T} \frac{\partial h_T}{\partial h_{1}} \frac{\partial^+ h_{1}}{\partial \phi} \right\| \rightarrow 0$ as $T\rightarrow +\infty$, i.e., the vanishing gradient problem occurs.
    	\end{theorem}
    	\begin{theorem}[Necessary condition of gradient exploding] \label{thm:exploding} If we observe the vanishing gradient problem, then at least one of the conditions listed in Theorem \ref{thm:explodingapp} holds.
    	\end{theorem}

    \section{Experimental Results} \label{sec:experiments}
        In this section, we present numerical results by comparing our algorithm with a GRU baseline model. The experiments are conducted on three datasets, and the source code is available at \url{http://after_accepted}.
        
        We test two versions of the RRNN algorithm. The first one is the full algorithm we presented in Algorithm \ref{algo:2}. The second one, which we call it RRNN-GRU, is a simplified version of the RRNN model where we limit the tree structure to be exactly the same as GRU. This model has a limited tree search space and the only dynamic component is the choice of the tuple $(L_i, R_i, b_i)$ to use on each pair of parent-child nodes, so the positioning of weights in the cell is flexible. Therefore, RRNN-GRU is still time-variant and data-dependent. In addition, we alternate between training the $L, R, b$ parameters and training the scoring neural network $\alpha$ consisting of a 2-layer fully connected neural network, while continuously training the output layer. The frequency (in epochs) that we switch training phases is set as a hyperparameter of the RRNN-GRU model. Due to the model architecture, training can sometimes be unstable with exploding gradients which we clip. The baseline model is the single layer GRU which has 100-dimensional hidden states.
        
        For both 100-dimensonal character and word embeddings, we used the pre-trained embedding vectors from GloVe\footnote{\url{https://nlp.stanford.edu/projects/glove/}}. The Adam optimizer is used for all experiments and random initial weights are selected. A random search on hyperparameters is used for all RRNN-GRU models and GRU models. We train the model parameters on the training set and select the optimal parameters and hyperparameters based on the performance measure on the validation set. Then we use this set of hyperparameters and the optimized model parameters to predict on the test set. We test the RRNN model only on the Wikipedia dataset. We report the performance on both validation and test sets for all datasets in Table \ref{table:performance} and list the optimal hyperparamers in Appendix \ref{appendix:hyperparameters}. Further details about the implementations are given in Appendix \ref{appendix:experimentsDetails}.
        \begin{table}
            \caption{Performance of models on three datasets}
            \label{table:performance}
            \centering
            \begin{tabular}{lcccc}
                \toprule
                \multirow{2}{*}{} &
                \multicolumn{2}{c}{Wiki-5k  (BPC)} & \multicolumn{2}{c}{Wiki-10k (BPC)}\\
                \cmidrule(lr){2-3} \cmidrule(lr){4-5} & Val & Test & Val & Test\\
                \midrule
                RRNN & \bf{2.58 (-5.5\%)} & \bf{2.63 (-1.9\%)} & -- & -- \\
                RRNN-GRU &-- & -- & \bf{2.43 (-5.8\%)} & \bf{2.42 (-5.8\%)} \\
                GRU & 2.73 & 2.68 & 2.58 & 2.57 \\
                \midrule
                \multirow{2}{*}{} &
                \multicolumn{2}{c}{SST (Accuracy)} & \multicolumn{2}{c}{PTB (Perplexity)}\\
                \cmidrule(lr){2-3} \cmidrule(lr){4-5} & Val & Test & Val & Test\\
                \midrule
                RRNN-GRU & \bf{65.1\% (-0.8\%)} & \bf{68.7\% (-5.2\%)} & 281 & 239  \\
                GRU & 64.6\% & 65.3 \% & \bf{247 (-12.1\%)} & 239 \\
                \bottomrule
            \end{tabular}
        \end{table}
        
        % \begin{table}
        %     \caption{Performance of models on three datasets}
        %     \label{table:performance}
        %     \centering
        %     \begin{tabular}{lcccccccc}
        %         \toprule
        %         \multirow{2}{*}{} &
        %         \multicolumn{2}{c}{Wiki-5k  (BPC)} & \multicolumn{2}{c}{Wiki-10k (BPC)} &

        %         \multicolumn{2}{c}{SST (Accuracy)} &
        %         \multicolumn{2}{c}{PTB (Perplexity)} \\
        %           \cmidrule(lr){2-3} \cmidrule(lr){4-5} \cmidrule(lr){6-7} \cmidrule(lr){8-9} & val & test & val & test & val & test & val & test \\
        %         \midrule
        %         RRNN & \bf{2.58 (-5.5\%)} & \bf{2.63 (-1.9\%)} & -- & -- & -- & -- & -- & -- \\
        %         RRNN-GRU &-- & -- & \bf{2.43 (-5.8\%)} & \bf{2.42 (-5.8\%)} & \bf{65.1}\% & \bf{68.7}\% & 279 & ? \\
        %         GRU & 2.73 & 2.68 & 2.58 & 2.57 & 64.6\% & 65.3 \% & ? & \bf{239} \\
        %         \bottomrule
        %     \end{tabular}
        % \end{table}
    
        \subsection{Datasets and Settings}
            The {\it{Wikipedia}} task is to predict the next character on text drawn from the Hutter prize Wikipedia dataset\footnote{\url{https://cs.fit.edu/~mmahoney/compression/textdata.html}} \cite{hutter2004universal}. We remove all numbers, punctuation, XML tags, and markup characters so that 26 English characters and space are left in the raw text. Performance is measured using BPC (the smaller the better). For RRNN-GRU, we randomly select 10,000 20-character sequences for the training set, along with 1,000 sequences for validation and 2,000 for testing, such that no sequences overlap. For RRNN, the training set has 5,000 sequences while the validation and test sets remain of the same size.
            
            The Stanford Sentiment Treebank (SST) dataset\footnote{\url{https://nlp.stanford.edu/sentiment/treebank.html}} \cite{socher2013recursive} is a sentiment analysis task involving classifying one-sentence movie reviews as positive, negative, or neutral. We obtain the dataset from the \texttt{torchtext} package and use the full 8,544-sample training set, along with a randomly-chosen 1,000 samples for validation and 2,000 for testing. Since the training data has variable length, we prepend each sample with zeros to make each sample be the same length. The performance is measured in the accuracy of correctly predicting sentiments (the higher the better).
                
            We also perform word-level language modeling using the \textit{Penn Treebank} (PTB) dataset\footnote{\url{https://catalog.ldc.upenn.edu/LDC99T42}} \cite{marcus1993building}, a corpus containing articles from the Wall Street Journal. We obtain this dataset from the \texttt{torchtext} package and randomly select a 10,000 sample subset of 20 words each, along with 1,000 samples for validation and 2,000 for testing. We predict over all 10,001 unique words in our subset without eliminating uncommon words. The performance is measured in perplexity (the smaller the better).
            
        \subsection{Discussion}
            From Table \ref{table:performance} it is clear that RRNN-GRU outperforms GRU by 5.8\% on the Wikipedia dataset while RRNN improves the results of GRU by 5.5\% on validation set and 1.9\% on test set with a simple set of hyperparameters. On SST, RRNN-GRU also beats GRU by 0.8\% and 5.2\% on validation and test sets, respectively. These experiments show that the data-dependent structures do help improve the prediction power of the model and achieve better performance. Meanwhile, RRNN-GRU matches the performance of GRU on PTB. Its performance on the test set of PTB can be improved by a more dedicated hyperparameter search. 

            One interesting observation of the full RRNN model is the evolution of the predicted tree structures. Figure \ref{fig:balancedtree} of Appendix \ref{appendix:figures} shows the common tree structures we find at the beginning epochs while Figure \ref{fig:imbalancedtree} of Appendix \ref{appendix:figures} shows the common tree structures at later epochs (near the point where optimal performance is achieved on the validation set). The tree structures tend to be balanced in the beginning epochs since the structure of the ground-truth tree plays a significant role. In later epochs, the output layer dominates the predicting ability and therefore the model tends to feed simple $h_t$ to the output layer.
            
            Another interesting observation lies in the dynamics of the RRNN-GRU model. Let us denote $\calI_{e,i,t,j}$ to be the index of parameter tuple $(L,R,b)$ that the $j$-th internal node of $i$-th sample on the $t$-th time step in $e$-th epoch, and we further set $N_e \triangleq \sum_{i,t,j}\mathbbm{1}\left\{ \calI_{e,i,t,j} \neq \calI_{e-1,i,t,j} \right\}$ to measure the number of changes in the choice of parameter tuples between $(e-1)$-th epoch and $e$-th epoch. Then we should expect the quantity $N_e$ to be decreasing as $e$ increases since the model is expected to become more stable as the training goes on and the choice of indices of parameter tuples should also become more stable. Figure \ref{fig:N_b} in Appendix \ref{appendix:figures} shows the plot of $N_e$ vs epochs which supports our hypothesis.
    
    \subsubsection*{Acknowledgments}
        The authors would like to acknowledge and thank Intel for providing access to Intel’s Computing environment.

    % \section{Conclusion}
    %     \xqcomment{Not reviewed yet}
        
    %     We presented RRNN, a simple yet efficient architecture search algorithm for recurrent neural networks. By recursively searching for data-dependent architectures in a huge tree structure space, RRNN promises to be particularly interesting by applying different structures for different time steps and samples. Experiments on real dataset suggest that RRNN (and its simplified version) is able to match or outperform the baseline GRU model.

    %     One future work is definitely be evaluating the RRNN model on more datasets. There are many interesting directions to improve RRNN further. For example, the current target tree structure is fixed as the GRU model. This could be improved by importing more well-designed recurrent cell structures as the ground-truth tree. Besides, it is also interesting to investigate the evolution of predicted tree structures through training epochs and therefore come up with new tree strucutures for general RNN models.

    \newpage
    \bibliographystyle{plain}
    \bibliography{ref}    

\begin{thebibliography}{10}

\bibitem{baker2016designing}
Bowen Baker, Otkrist Gupta, Nikhil Naik, and Ramesh Raskar.
\newblock Designing neural network architectures using reinforcement learning.
\newblock In {\em International Conference on Learning Representations (ICLR)},
  2017.

\bibitem{279181}
Yoshua Bengio, Patrice Simard, and Paolo Frasconi.
\newblock Learning long-term dependencies with gradient descent is difficult.
\newblock {\em IEEE Transactions on Neural Networks}, 5(2):157--166, 1994.

\bibitem{cai2017making}
Jonathon Cai, Richard Shin, and Dawn Song.
\newblock Making neural programming architectures generalize via recursion.
\newblock In {\em International Conference on Learning Representations (ICLR)},
  2017.

\bibitem{chollet2017xception}
Fran{\c{c}}ois Chollet.
\newblock Xception: Deep learning with depthwise separable convolutions.
\newblock In {\em Proceedings of the IEEE Conference on Computer Vision and
  Pattern Recognition (CVPR)}, pages 1251--1258, 2017.

\bibitem{chung2014empirical}
Junyoung Chung, Caglar Gulcehre, Kyunghyun Cho, and Yoshua Bengio.
\newblock Empirical evaluation of gated recurrent neural networks on sequence
  modeling.
\newblock In {\em Conference on Neural Information Processing Systems (NIPS)
  Workshop on Deep Learning}, 2014.

\bibitem{de2015survey}
Wim De~Mulder, Steven Bethard, and Marie-Francine Moens.
\newblock A survey on the application of recurrent neural networks to
  statistical language modeling.
\newblock {\em Computer Speech \& Language}, 30(1):61--98, 2015.

\bibitem{elsken2018neural}
Thomas Elsken, Jan~Hendrik Metzen, and Frank Hutter.
\newblock Neural architecture search: A survey.
\newblock {\em arXiv preprint arXiv:1808.05377}, 2018.

\bibitem{goller1996learning}
Christoph Goller and Andreas Kuchler.
\newblock Learning task-dependent distributed representations by
  backpropagation through structure.
\newblock In {\em Proceedings of International Conference on Neural Networks
  (ICNN)}, pages 347--352, 1996.

\bibitem{graves2016adaptive}
Alex Graves.
\newblock Adaptive computation time for recurrent neural networks.
\newblock {\em arXiv preprint arXiv:1603.08983}, 2016.

\bibitem{graves2005framewise}
Alex Graves and J{\"u}rgen Schmidhuber.
\newblock Framewise phoneme classification with bidirectional {LSTM} and other
  neural network architectures.
\newblock {\em Neural Networks}, 18(5-6):602--610, 2005.

\bibitem{hochreiter1997long}
Sepp Hochreiter and J{\"u}rgen Schmidhuber.
\newblock Long short-term memory.
\newblock {\em Neural Computation}, 9(8):1735--1780, 1997.

\bibitem{hopfield1982neural}
John~J. Hopfield.
\newblock Neural networks and physical systems with emergent collective
  computational abilities.
\newblock {\em Proceedings of the National Academy of Sciences},
  79(8):2554--2558, 1982.

\bibitem{hutter2004universal}
Marcus Hutter.
\newblock {\em Universal artificial intelligence: Sequential decisions based on
  algorithmic probability}.
\newblock Springer Science \& Business Media, 2004.

\bibitem{kalchbrenner2015grid}
Nal Kalchbrenner, Ivo Danihelka, and Alex Graves.
\newblock Grid long short-term memory.
\newblock {\em arXiv preprint arXiv:1507.01526}, 2015.

\bibitem{li2015constructing}
Xiangang Li and Xihong Wu.
\newblock Constructing long short-term memory based deep recurrent neural
  networks for large vocabulary speech recognition.
\newblock In {\em IEEE International Conference on Acoustics, Speech and Signal
  Processing (ICASSP)}, pages 4520--4524. IEEE, 2015.

\bibitem{lipton2015critical}
Zachary~C. Lipton, John Berkowitz, and Charles Elkan.
\newblock A critical review of recurrent neural networks for sequence learning.
\newblock {\em arXiv preprint arXiv:1506.00019}, 2015.

\bibitem{liu2018darts}
Hanxiao Liu, Karen Simonyan, and Yiming Yang.
\newblock {DARTS}: Differentiable architecture search.
\newblock In {\em International Conference on Learning Representations (ICLR)},
  2019.

\bibitem{marcus1993building}
Mitchell Marcus, Beatrice Santorini, and Mary~Ann Marcinkiewicz.
\newblock Building a large annotated corpus of english: The penn treebank.
\newblock 1993.

\bibitem{miller1995neural}
W.~Thomas Miller, Paul~J. Werbos, and Richard~S. Sutton.
\newblock {\em Neural networks for control}.
\newblock MIT press, 1995.

\bibitem{pascanu2012understanding}
Razvan Pascanu, Tomas Mikolov, and Yoshua Bengio.
\newblock On the difficulty of training recurrent neural networks.
\newblock In {\em International Conference on Machine Learning (ICML)}, pages
  1310--1318, 2013.

\bibitem{pham2018efficient}
Hieu Pham, Melody Guan, Barret Zoph, Quoc~V. Le, and Jeff Dean.
\newblock Efficient neural architecture search via parameter sharing.
\newblock In {\em International Conference on Machine Learning (ICML)}, pages
  4092--4101, 2018.

\bibitem{sak2014long}
Ha{\c{s}}im Sak, Andrew Senior, and Fran{\c{c}}oise Beaufays.
\newblock Long short-term memory recurrent neural network architectures for
  large scale acoustic modeling.
\newblock In {\em Proceedings of the Annual Conference of the International
  Speech Communication Association (INTERSPEECH)}, pages 338--342, 2014.

\bibitem{schrimpf2017flexible}
Martin Schrimpf, Stephen Merity, James Bradbury, and Richard Socher.
\newblock A flexible approach to automated {RNN} architecture generation.
\newblock {\em arXiv preprint arXiv:1712.07316}, 2017.

\bibitem{socher2011parsing}
Richard Socher, Cliff~C. Lin, Christopher~D. Manning, and Andrew Ng.
\newblock Parsing natural scenes and natural language with recursive neural
  networks.
\newblock In {\em International Conference on Machine Learning (ICML)}, pages
  129--136, 2011.

\bibitem{socher2013recursive}
Richard Socher, Alex Perelygin, Jean Wu, Jason Chuang, Christopher~D. Manning,
  Andrew Ng, and Christopher Potts.
\newblock Recursive deep models for semantic compositionality over a sentiment
  treebank.
\newblock In {\em Proceedings of the 2013 Conference on Empirical Methods in
  Natural Language Processing (EMNLP)}, pages 1631--1642. Association for
  Computational Linguistics, 2013.

\bibitem{yu2015multi}
Fisher Yu and Vladlen Koltun.
\newblock Multi-scale context aggregation by dilated convolutions.
\newblock In {\em International Conference on Learning Representations (ICLR)},
  2016.

\bibitem{zhang2018layer}
Lida Zhang and Diego Klabjan.
\newblock Layer flexible adaptive computational time for recurrent neural
  networks.
\newblock {\em arXiv preprint arXiv:1812.02335}, 2018.

\bibitem{zoph2016neural}
Barret Zoph and Quoc~V. Le.
\newblock Neural architecture search with reinforcement learning.
\newblock In {\em International Conference on Learning Representations (ICLR)},
  2017.

\bibitem{zoph2018learning}
Barret Zoph, Vijay Vasudevan, Jonathon Shlens, and Quoc~V. Le.
\newblock Learning transferable architectures for scalable image recognition.
\newblock In {\em Proceedings of the IEEE Conference on Computer Vision and
  Pattern Recognition (CVPR)}, pages 8697--8710, 2018.

\end{thebibliography}
    
    \newpage
    \appendix
    \appendixpage
    \addappheadtotoc

    \begin{appendices}
    \section{Figures} \label{appendix:figures}
    	\begin{figure}[!h]
		    \centering
		    \includegraphics[width=.5\textwidth]{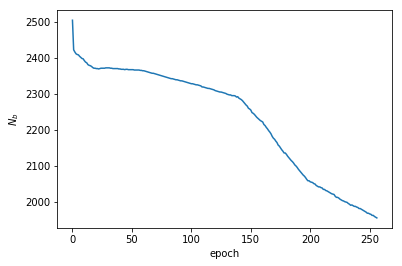}
		    \caption{The number of changes in the indices of parameter tuples vs epochs}
			\label{fig:N_b}
		\end{figure}
		
        \begin{figure}[!h]
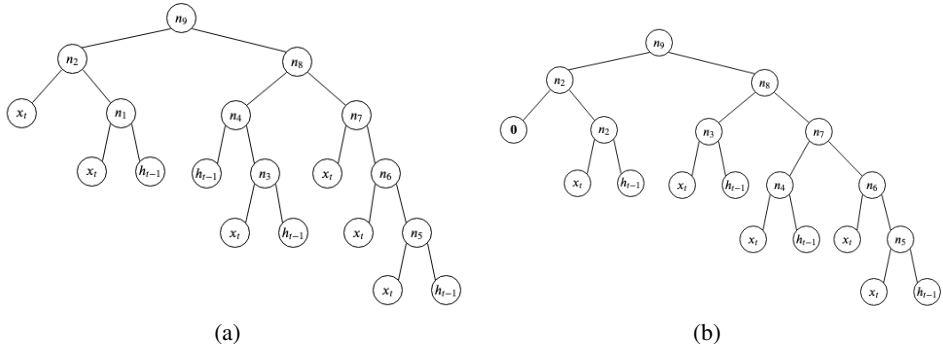

            \centering
            \subfigure[]{\includegraphics[width=.45\textwidth]{t1.png}}
            \subfigure[]{\includegraphics[width=.45\textwidth]{t2.png}}
            \caption{Example of tree structures at early stage of training.}
            \label{fig:balancedtree}
        \end{figure}
        
        \begin{figure}[!h]
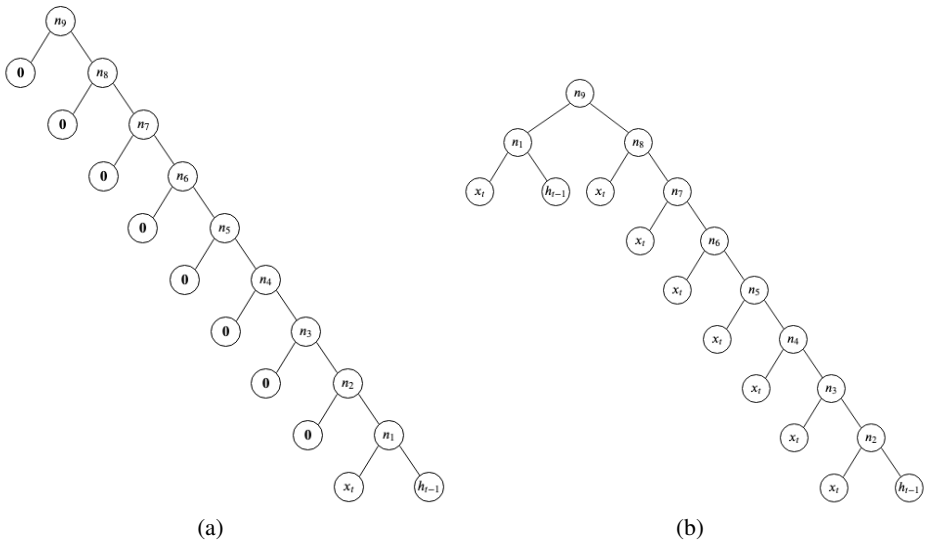

            \centering
            \subfigure[]{\includegraphics[width=.45\textwidth]{t3.png}}
            \subfigure[]{\includegraphics[width=.45\textwidth]{t4.png}}
            \caption{Example of tree structures near the optimal stage of training.}
            \label{fig:imbalancedtree}
        \end{figure}
		
    \newpage
    \section{Extensions of RRNN model} \label{sec:extension} 
	    The exposition so far handles the case where only one state vector transfers between hidden cells of the RRNN model, and it can capture the structure of GRU. However, LSTM, for example, has two state vectors $h_t$ and $c_t$ to transfer between cells. In this section we extend the RRNN model to be compatible with transferring multiple state vectors. Suppose that a total of $M$ vectors, $h_{t-1, 1}, \ldots, h_{t-1, M}$, are the output of the $(t-1)$-th hidden cell. The transition equations and cell output are thereby 
	    \begin{align*}
    	    &(T_{t, i}^{\pred}, h_{t, i}) = f^i(\calN_{0, i}^t) \quad i=1,2,\ldots,M, \\
    	    &q_t = g(x_t, h_{t, M}; \Gamma),
	    \end{align*}
	    where each $f^i \triangleq f_N^i \circ \cdots \circ f_1^i$ has the same definition as the function $f$ defined in (\ref{eq:f}), and $\calN_{0, i}^t$ is the multiset consisting of multiple copies of $x_t$, $h_{t-1, j}, j=1, \ldots, M$, $h_{t, j}, j=1, \ldots, i-1$, and possible other constant vectors. In practice, we use Algorithm \ref{algo:2} $M$ times to build functions $f^i, i=1,2,\ldots, M$. 
	   
	    The loss function is redefined as \begin{align*}
	        L(\Phi)=\mathbb{E}_{(X,Y)}\Bigg[\sum_{t=1}^T \bigg\{ \lambda_1 l(y_t,q_t)& +\lambda_2 \sum_{i=1}^M \min_{\bar{T}\in {\Iso}(T_{t, i}^{\target})} \mathrm{TD}(\bar{T},T_{t, i}^{\pred}) +\lambda_3 \sum_{i=1}^M\sum_{k=0}^{N-1} m(\calN_{k,i}^t)\bigg\}\Bigg] \\ & + \lambda_4 \sum_{\phi \in \Phi} \|\phi \|^2 ,
	    \end{align*}  
	    where $\calN_{k,i}^t = f_k^i \circ \cdots \circ f_1^i (\calN_{0, i}^t), k=1, 2, \ldots, N-1, i=1,2,\ldots, M$, and $T_{t,i}^{\target}$ is the ground truth binary tree.
	    
	    As an example, we show how to transfer two state vectors $c_t$ and $h_t$ between hidden cells of RRNN and mimic the structure of LSTM. To adhere with notation from prior works, we use $c_t$ and $h_t$ to replace $h_{t,1}$ and $h_{t,2}$ in the above general definition. The transition equations and cell output are therefore \begin{align*}
    	    &(T_{t, 1}^{\pred}, c_t) = f^1(\calN_{0, 1}^t(x_t, c_{t-1}, h_{t-1})), \\
    	    &(T_{t, 2}^{\pred}, h_t) = f^2(\calN_{0, 2}^t(x_t, c_{t-1}, h_{t-1}, c_t)), \\
    	    &q_t = g(x_t,h_{t}; \Gamma),
	    \end{align*}
	    where $\calN_{0, 1}^t(x_t, c_{t-1}, h_{t-1})$ consists of several copies of $x_t, h_{t-1}, c_{t-1}$ and possible other constant vectors, and $\calN_{0, 2}^t(x_t, c_{t-1}, h_{t-1}, c_t)$ consists of several copies of $x_t, c_{t-1}, h_{t-1}, c_t$ and possible other constant vectors.

    \section{Expressibility of RRNN} \label{appendix:expressibility}
        We extend the contect of Section \ref{sec:experiments} here. We first show that for a given set of vectors, there always exists a scoring function that can rank the scores of these vectors by any order we want. Formally, we have the following lemma.
    	\begin{lemma} \label{lemma:scoring}
    		Given $ n $ vectors $ v_1, \ldots, v_n \in \mathbb{R}^p $, there exists a function $ \alpha $ with a set of parameters $ \Theta $ such that $\alpha(v_1; \Theta) > \cdots > \alpha(v_n; \Theta)$.
    	\end{lemma}
    	
    	We next show that if we carefully choose sets $\calL, \calR, \calB, \calU, \calO, \mathcal{S}_t^{data}, \mathcal{S}_t^{prev},  \mathcal{S}_t^{aux}$, the quantity $N$, and the scoring function $\alpha$, then Algorithm \ref{algo:2} can replicate the GRU and LSTM equations. 
    	
    	To replicate GRU, we should have $N=8$, $ \mathcal{S}_t^{data} = \{x_t\}$, $ \mathcal{S}_t^{prev} = \{ h_{t-1}\}$, $ \mathcal{S}_t^{aux} = \{ \bf{0}\}$, where $\bf{0}$ is the zero vector. We further set 
        \begin{itemize}
			\item $ \calL = \{L_1, L_2, L_3, L_4\}$, where $L_1 = W_r, L_2 = W_z, L_3 = W_h$, and $L_4=I$,
			\item $ \calR = \{R_1, R_2, R_3, R_4\}$, where $R_1 = W_r^\prime, R_2 = W_z^\prime, R_3 = W_h^\prime$, and $R_4=I$,
			\item $ \calB = \{b_1, b_2, b_3, b_4\}$, where $b_1 = b_r, b_2 = b_z, b_3 = b_h$, and $b_4=\bf{0}$,
			\item $ \mathcal{U} = \left\{ \sigma(\cdot), \tanh(\cdot), \mathbf{1}-\cdot, \mathrm{id}(\cdot) \right\} $, where $\mathbf{1}$ stands for the all-ones vector and $\mathrm{id}$ stands for the identity mapping,
			\item $\mathcal{O} = \left\{ +, \odot \right\} $, where $\odot$ is the entry-wise multiplication.
		\end{itemize}
    	
    	Theorem \ref{thm:GRU} therefore becomes the following
    	
    	\begin{theorem} \label{thm:GRUapp}
    		There exists a scoring function $ \alpha $ such that Algorithm \ref{algo:2} generates GRU equations (\ref{eq:GRUeq1}) -- (\ref{eq:GRUeq4}) for the choice of $\calL, \calR, \calB, \calU, \calO, \mathcal{S}_t^{data}, \mathcal{S}_t^{prev}$, and $  \mathcal{S}_t^{aux}$ specified above.
    	\end{theorem}
    	
    	For LSTM, note that there are two state vectors $h_t$ and $c_t$. Therefore, to replicate LSTM (see equations (\ref{eq:LSTMeq1}) -- (\ref{eq:LSTMeq5}) below), we run Algorithm \ref{algo:2} twice. In the first run, we should have $N=7$, $ \mathcal{S}_t^{data} = \{x_t\}$, $ \mathcal{S}_t^{prev} = \{ h_{t-1}, c_{t-1}\}$, $ \mathcal{S}_t^{aux} = \{ \bf{0}\}$. We further set 
            \begin{itemize}
    			\item $ \calL = \{L_6\}$, where $L_6 = W_f, L_2 = W_i, L_3 = W_o, L_4 = W_c$, and $L_5=I$,
    			\item $ \calR = \{R_1, R_2, R_3, R_4, R_5\}$, where $R_1 = W_f^\prime, R_2 = W_i^\prime, R_3 = W_o^\prime, R_4 = W_c^\prime$, and $R_5=I$,
    			\item $ \calB = \{b_1, b_2, b_3, b_4, b_5\}$, where $b_1 = b_f, b_2 = b_i, b_3 = b_o, b_4 = b_c$, and $b_5=\bf{0}$,
    			\item $ \mathcal{U} = \left\{ \sigma(\cdot), \tanh(\cdot), \mathrm{id}(\cdot) \right\} $,
    			\item $\mathcal{O} = \left\{ +, \odot \right\} $.
    		\end{itemize}
    		
    	In the second run, we should have $N=2$, $ \mathcal{S}_t^{data} = \{x_t\}$, $ \mathcal{S}_t^{prev} = \{ h_{t-1}, c_{t-1}, c_t\}$, $ \mathcal{S}_t^{aux} = \{ \bf{0}\}$. We further set $\calL = \{L_6\}, \calR = \{R_6\}$, where $L_6 = R_6 = I$, $\calB = \{b_6\}$, where $b_6 = \bf{0}$, $\calU = \{\tanh(\cdot), \mathrm{id}(\cdot)\}$, and $\calB = \{+, \odot\}$. Theorem \ref{thm:LSTM} therefore becomes the following
        \begin{theorem} \label{thm:LSTMapp}
        	There exists a scoring function $ \alpha $ such that Algorithm \ref{algo:2} (applied twice) generates the following LSTM equations \begin{align}
        		f_t &= \sigma \left(W_f x_{t} + W_f' h_{t-1} + b_f\right) \label{eq:LSTMeq1}\\
        		i_t &= \sigma \left(W_i x_{t} + W_i' h_{t-1} + b_i\right)\\
        		o_t &= \sigma \left(W_o x_{t} + W_o' h_{t-1} + b_o\right)\\
        		c_t &= c_{t-1} \odot f_t + i_t \odot \tanh\left(W_c x_{t} + W_c' h_{t-1} + b_c\right) \label{eq:LSTMeq4}\\
        		h_t &= o_t \odot \tanh(c_t)	\label{eq:LSTMeq5}
    		\end{align}
    		for the choice of $\calL, \calR, \calB, \calU, \calO, \mathcal{S}_t^{data}, \mathcal{S}_t^{prev}$, and $  \mathcal{S}_t^{aux}$ specified above. 
        \end{theorem}

        \subsection{Proof of Lemma \ref{lemma:scoring}} \label{appendix:lemma1}
    		Consider function \begin{equation*}
    		    \alpha(v; \Theta) = \sum_{k=1}^{n} k \exp\left(-\frac{\|v-v_k\|^2}{2\sigma_0^2}\right)
    		\end{equation*}
    		with $ \Theta = \{\sigma_0\} $, and $\sigma_0$ is a large enough constant such that $\sigma_0^2 \ge \frac{\Delta}{\log(n^2)-\log(n^2-1)}$, and $ \Delta = \frac{1}{2}\min_{j\neq k} \|v_j-v_k\|^2$.
    		
    		For $0 \le i \le n-1$, since $-\frac{\|v_{i+1}-v_k\|^2}{2\sigma_0^2}\le 0$, we have $$ \alpha(v_{i+1}; \Theta) = \sum_{k=1}^{n} k \exp\left(-\frac{\|v_{i+1}-v_k\|^2}{2\sigma_0^2}\right) \le \sum_{\substack{1\le k \le n \\ k\neq {i+1}}} k = \frac{n(n+1)}{2}-(i+1)\,.$$
    		
    		On the other hand, for a fixed $1\le i \le n$ and all $1\le k \le n$, we have $\exp\left(-\frac{\|v_i-v_k\|^2}{2\sigma_0^2}\right) \ge \exp\left(-\frac{\Delta}{\sigma_0^2}\right) \ge \exp\left( \log(n^2-1)-\log(n^2)\right) = 1-\frac{1}{n^2} \ge 1-\frac{1}{nk} $, and therefore \begin{align*}
    		    \alpha(v_i; \Theta) &= \sum_{k=1}^{n} k \exp\left(-\frac{\|v_i-v_k\|^2}{2\sigma_0^2}\right) \ge \sum_{\substack{1\le k \le n \\ k\neq i}} k\left(1-\frac{1}{nk}\right) \\
    		    &= \frac{n(n+1)}{2}-i-\frac{n-1}{n} > \frac{n(n+1)}{2}-(i+1)\,.
    		\end{align*}
    
    		In conclusion, for $1\le i \le n-1$, we have  $$\alpha(v_{i+1}; \Theta)\le \frac{n(n+1)}{2}-(i+1) < \alpha(v_i; \Theta)\,,$$
    		and thus $\alpha(v_1; \Theta) > \cdots > \alpha(v_n; \Theta)$.
    		
        \subsection{Proof of Theorem \ref{thm:GRUapp} and \ref{thm:LSTMapp}}
            We start with the proof of Theorem \ref{thm:GRUapp}. In Algorithm \ref{algo:2}, we set $N=8$, $ \mathcal{S}_t^{data} = \{x_t\}$, $ \mathcal{S}_t^{prev} = \{ h_{t-1}\}$, $ \mathcal{S}_t^{aux} = \{ \bf{0}\}$, where $\bf{0}$ is the zero vector. In the following, ``the algorithm'' refers to Algorithm \ref{algo:2}. The scoring function $\alpha$ has a set of parameters $\Theta$ and is capable of sorting the scores of different vectors. We show the existence of this function at the end of the proof by relying on Lemma \ref{lemma:scoring}. 
        	
        	We start with $\calP_0 = \emptyset$. For $k=1$, the algorithm generates the vector set $ V_t^1 $ and one of its elements is $ r_t \triangleq \sigma \left(L_1 x_{t} + R_1 h_{t-1} + b_1\right) = \sigma \left(W_r x_{t} + W_r^\prime h_{t-1} + b_r\right) $. The scoring function $\alpha$ guarantees that $\alpha(r_t; \Theta) > \alpha(v; \Theta), \forall v \in V_t^1, v\neq r_t$. Therefore, we have $c_1^* = r_t$ and $\calP_1 = \{r_t\}$. Similarly, the algorithm finds $z_t$ to be the vector with the highest score in the set $V_t^2\setminus\calP_1$. We have $c_2^* = z_t$ and $\calP_2 = \{r_t, z_t\}$.
        	
        	For $k=3$, the algorithm generates the vector set $ V_t^3 $ and one of its elements is $ \widetilde{r}_t \triangleq \mathrm{id}[(L_4 h_{t-1}) \odot (R_4 r_t) + b_4] = r_t \odot h_{t-1}$. The scoring function $\alpha$ guarantees that $\alpha(\widetilde{r}_t; \Theta) > \alpha(v; \Theta), \forall v \in V_t^3\setminus\calP_2, v\neq \widetilde{r}_t$. Therefore, we have $c_3^* =\widetilde{r}_t$ and $\calP_3 = \{r_t, z_t, \widetilde{r}_t\}$.
        	
        	For $k=4$, the algorithm generates the vector set $ V_t^4 $ and one of its elements is $ \mathbf{1} - z_t = \mathbf{1} - (L_4 \mathbf{0} + R_4 z_{t} + b_4)$. The scoring function $\alpha$ guarantees that $\alpha(\mathbf{1}-z_t; \Theta) > \alpha(v; \Theta), \forall v \in V_t^4\setminus\calP_3, v\neq \mathbf{1}-z_t$. Therefore, we have $c_4^* = \mathbf{1}-z_t$ and $\calP_4 = \{r_t, z_t, \widetilde{r}_t, \mathbf{1}-z_t\}$.
        	
        	For $k=5$, the algorithm generates the vector set $ V_t^5 $ and one of its elements is $ \widetilde{h}_t \triangleq \tanh(L_3 x_{t} + R_3 \widetilde{r}_t + b_3) = \tanh(W_h x_{t} + W_h^\prime (r_t \odot h_{t-1}) + b_h) $. The scoring function $\alpha$ guarantees that $\alpha(\widetilde{h}_t; \Theta) > \alpha(v; \Theta), \forall v \in V_t^5\setminus\calP_4, v\neq \widetilde{h}_t$. Therefore, we have $c_5^* =\widetilde{h}_t$ and $\calP_5 = \{r_t, z_t, \widetilde{r}_t, \mathbf{1}-z_t, \widetilde{h}_t\}$.
        	
        	For $k=6$, the algorithm generates the vector set $ V_t^6 $ and one of its elements is $ z_t \odot h_{t-1} = \mathrm{id}[(L_4 h_{t-1}) \odot (R_4 z_{t}) + b_4]$. The scoring function $\alpha$ guarantees that $\alpha(z_t \odot h_{t-1}; \Theta) > \alpha(v; \Theta), \forall v \in V_t^6\setminus\calP_5, v\neq z_t \odot h_{t-1}$. Therefore, we have $c_6^* = z_t \odot h_{t-1}$ and $\calP_6 = \{r_t, z_t, \widetilde{r}_t, \mathbf{1}-z_t, \widetilde{h}_t, z_t \odot h_{t-1}\}$. Similarly, the algorithm finds $(\mathbf{1}-z_t)\odot \widetilde{h}_t$ to be the vector with the highest score in the set $V_t^7\setminus\calP_6$. Thus we have $c_7^* = (\mathbf{1}-z_t)\odot \widetilde{h}_t$ and $ \calP_7 = \{r_t, z_t, \widetilde{r}_t, \mathbf{1}-z_t, \widetilde{h}_t, z_t \odot h_{t-1}, (\mathbf{1}-z_t)\odot \widetilde{h}_t\}$
        	
        	Finally, for $k=8$, the algorithm generates the vector set $ V_t^8 $ and one of its elements is $ h_t \triangleq \mathrm{id}[(L_4 (z_t\odot h_{t-1})) + (R_4 ((1-z_t)\odot \widetilde{h}_t)) + b_4] = z_t\odot h_{t-1} + (1-z_t)\odot \widetilde{h}_t$. The scoring function $\alpha$ guarantees that $\alpha(h_t; \Theta) > \alpha(v; \Theta), \forall v \in V_t^8\setminus\calP_7, v\neq h_t$. Therefore, we have $c_8^* = h_t$ and $\calP_8 = \{r_t, z_t, \widetilde{r}_t, \mathbf{1}-z_t, \widetilde{h}_t, z_t \odot h_{t-1}, (\mathbf{1}-z_t)\odot \widetilde{h}_t, h_t\}$.
    		
            It remains to specify the scoring function $\alpha$. Note that $\alpha$ should satisfy that $\alpha(c_i^*; \Theta) > \alpha(v; \Theta), \forall v \in V_t^i\setminus\calP_{i-1}, v \neq c_i^*$ for $1 \le i \le 8$. Since each set $V_t^i$ contains a finite number of vectors, Lemma \ref{lemma:scoring} guarantees that such scoring function $\alpha$ exists. Therefore, the algorithm returns vector $h_t$ and the binary tree rooted at $h_t$ as the output, and thus it replicates the GRU equations (\ref{eq:GRUeq1}) -- (\ref{eq:GRUeq4}). \qed
    
            The proof of Theorem \ref{thm:LSTMapp} is almost the same as the proof above. The only difference is that in the first run of Algorithm \ref{algo:2}, we generate equations (\ref{eq:LSTMeq1}) -- (\ref{eq:LSTMeq4}), while in the second run of Algorithm \ref{algo:2}, we generate equation (\ref{eq:LSTMeq5}).
    
    \section{Gradient Control} \label{appendix:gradientcontrol}
        In this section, we first give the full statements of Theorem \ref{thm:vanishing} and \ref{thm:exploding}. Then we give proofs for these two theorems and discuss about them.
        
    	\begin{theorem}[Sufficient condition of gradient vanishing] \label{thm:vanishingapp} Let $\calV_t$ to be the set of vectors on the nodes of predicted tree $T_t^\pred$. Assume that there exist constants $C_1, C_2, C_3, C_4$ such that for all $1\le t \le T$,
    	    \begin{align}
                &\|L\| \le C_1, \forall L\in \calL, \|R\| \le C_1, \forall R\in \calR, \label{eq:vanishCond1} \\
                &\|u^\prime\|_{\infty} \le C_2, \forall u \in \calU, \label{eq:vanishCond2} \\
                &\left\|\frac{\partial o(L_i v_1, R_iv_2)}{\partial (L_i v_1)}\right\| \le C_3, \left\|\frac{\partial o(L_i v_1, R_iv_2)}{\partial (R_i v_2)}\right\| \le C_3,  1\le i \le n_l, \forall v_1, v_2 \in \calV_t, v_1 \neq v_2, \forall o \in \mathcal{O}, \label{eq:vanishCond3} \\
                & \left\|\frac{\partial \calE_t}{\partial h_t}\right\| \le C_4, t = 0,1, \dots, \label{eq:vanishCond4} \\
                &C_1 C_2 C_3 < \frac{1}{2}. \label{eq:vanishCond5} 
            \end{align}
    	    Under conditions (\ref{eq:vanishCond1}) -- (\ref{eq:vanishCond5}), we have $\left\| \frac{\partial \calE_T}{\partial h_T} \frac{\partial h_T}{\partial h_{1}} \frac{\partial^+ h_{1}}{\partial \phi} \right\| \rightarrow 0$ as $T\rightarrow +\infty$, i.e., the vanishing gradient problem occurs.
    	\end{theorem}
    	
    	\begin{theorem}[Necessary condition of gradient exploding, restated] \label{thm:explodingapp}
        	    Let $l_{\min} \triangleq \lfloor \log_2(N+1)\rfloor +1$ be the minimum possible depth of all full binary trees $T_t^\pred, 1\le t \le T$. If the exploding gradient problem occurs, then at least one of the following conditions hold:
        	   % {\color{white} \;} \quad (i) there exists an activation function $u\in \calU$ such that $\|u^\prime\| \ge (N+1)^{-\frac{1}{3l_{\min}}}$, \\
            %     {\color{white} \;} \quad (ii) there exists a parameter matrix $P \in \calL\cup\calR$ such that $\|P\| \ge (N+1)^{-\frac{1}{3l_{\min}}}$,\\
            %     {\color{white} \;} \quad (iii) for infinite many $t$, there exists a pair of parent-child nodes $(v, v_1)$ in the tree $T_t^\pred$ such that $\left\|\frac{\partial o(L_i v_1, R_iv_2)}{\partial (L_i v_1)}\right\| \ge (N+1)^{-\frac{1}{3l_{\min}}}$, where $v = u(o(L_iv_1, R_iv_2)+b_i)$, \\
            %     {\color{white} \;} \quad (iv) for infinite many $t$, there exists a pair of parent-child nodes $(v, v_2)$ in the tree $T_t^\pred$ such that $\left\|\frac{\partial o(L_i v_1, R_iv_2)}{\partial (R_i v_2)}\right\| \ge (N+1)^{-\frac{1}{3l_{\min}}}$, where $v = u(o(L_iv_1, R_iv_2)+b_i)$.
        	    \begin{itemize}[noitemsep]
                    \item there exists an activation function $u\in \calU$ such that $\|u^\prime\| \ge (N+1)^{-\frac{1}{3l_{\min}}}$,
                    \item there exists a parameter matrix $P \in \calL\cup\calR$ such that $\|P\| \ge (N+1)^{-\frac{1}{3l_{\min}}}$,
                    \item for infinite many $t$, there exists a pair of parent-child nodes $(v, v_1)$ in the tree $T_t^\pred$ such that $\left\|\frac{\partial o(L_i v_1, R_iv_2)}{\partial (L_i v_1)}\right\| \ge (N+1)^{-\frac{1}{3l_{\min}}}$, where $v = u(o(L_iv_1, R_iv_2)+b_i)$,
                    \item for infinite many $t$, there exists a pair of parent-child nodes $(v, v_2)$ in the tree $T_t^\pred$ such that $\left\|\frac{\partial o(L_i v_1, R_iv_2)}{\partial (R_i v_2)}\right\| \ge (N+1)^{-\frac{1}{3l_{\min}}}$, where $v = u(o(L_iv_1, R_iv_2)+b_i)$.
                \end{itemize}        	    
                
            \end{theorem}
    	
        \subsection{Proof of Theorem \ref{thm:vanishingapp}}
    	    Recall that Algorithm \ref{algo:2} builds a full binary tree $T_t^{\pred}$ with $N$ internal nodes and $N+1$ leaf nodes for the $t$-th hidden cell of the RRNN model. We use $n_{1, t}^{\mathrm{int}}, \ldots, n_{N, t}^{\mathrm{int}}$ and $n_{1, t}^{\mathrm{leaf}}, \ldots, n_{N+1, t}^{\mathrm{leaf}}$ to denote the internal nodes and leaf nodes of the tree $T_t^\pred$, respectively. We denote $\calV_t = \{v_{n_{1, t}^{\mathrm{int}}}, \ldots, v_{n_{N, t}^{\mathrm{int}}},v_{n_{1, t}^{\mathrm{leaf}}}, \ldots,v_{ n_{N+1, t}^{\mathrm{leaf}}}\}$ to be the set of vectors on the predicted tree $T_t^\pred$. We use $\|A\|$ and $\|v\|_\infty$ to denote the spectral norm of matrix $A$ and the infinity norm of vector $v$, respectively. We use $\diag\{v\}$ to denote the diagonalization of vector $v$. For an activation function $u \in \calU$, we use $u^\prime$ to denote the derivative of $u$.

    	Note that $$
    	    \frac{\partial \calE_T}{\partial \phi} = \sum_{t^\prime=1}^T \frac{\partial \calE_T}{\partial h_T} \frac{\partial h_T}{\partial h_{t^\prime}} \frac{\partial^+ h_{t^\prime}}{\partial \phi} = \sum_{t^\prime=1}^T \frac{\partial \calE_T}{\partial h_T} \left( \prod_{t^\prime < t \le T} \frac{\partial h_t}{\partial h_{t-1}} \right) \frac{\partial^+ h_{t^\prime}}{\partial \phi}.
    	$$
    	
    	Intuitively, the vanishing gradients problem appears when the norm of $\frac{\partial h_i}{\partial h_{t-1}}$ is smaller than 1. We first provide some lemmas that facilitate proving Theorem \ref{thm:vanishing}. For simplicity, we remove the subscript $t$ in $n_{k, t}^{\mathrm{int}}$ and $n_{k, t}^{\mathrm{leaf}}$, since the following derivation applies to all $1\le t \le T$ in the same way.
    
        We define the path starting from the root node $ n_{N}^{\mathrm{int}} $ to a leaf node $ n_k^{\mathrm{leaf}} $ by $ P^k = \big[P_0^k, P_1^k, \ldots, P_{l_k}^k\big] $, where $ l_k $ is the length of this path, $P_0^k = n_k^{\mathrm{leaf}}$, and $P_{l_k}^k=n_{N}^{\mathrm{int}}$. Lemma \ref{lemma:nodechild} gives an upper-bound for the norm of the gradient of a node with respect to one of its child node in the binary tree $T_t^{\pred}$.
    	\begin{lemma} \label{lemma:nodechild}
    	    Under conditions (\ref{eq:vanishCond1}) -- (\ref{eq:vanishCond4}), there exists a constant $C_0 < \frac{1}{2}$ such that $\left\| \frac{\partial P_j^k}{\partial P_{j-1}^k}\right\| \le C_0$ for all $1\le k \le N+1$ and $1\le j \le l_k$.
    	\end{lemma} \begin{proof}
           For simplicity, we write $v = u(o(L_iv_1, R_iv_2)+b_i)$, where $i$ is the index in $\calL, \calR$ and $\calB$, $v=P_j^k$, $v_1 = P_{j-1}^k$, and $v_2$ is the other child node of $P_j^k$. The case where $v_2 = P_{j-1}^k$ is similar.
            
            By the chain rule, we have \begin{align*}
                \left\|\frac{\partial v}{\partial v_1}\right\| &= \left\| \frac{\partial u(o(L_i v_1, R_iv_2)+b_i)}{\partial [o(L_i v_1, R_iv_2)+b_i]} \frac{\partial [o(L_i v_1, R_iv_2)+b_i]}{\partial (L_i v_1)} \frac{\partial (L_i v_1)}{\partial v_1} \right\| \\
                &= \left\|\diag\left\{u^\prime(o(L_iv_1, R_iv_2)+b_i)\right\} \frac{\partial o(L_iv_1, R_iv_2)}{\partial (L_iv_i)} L_i\right\| \\
                &\le \left\|\diag\left\{u^\prime(o(L_iv_1, R_iv_2)+b_i)\right\} \right\| \left\| \frac{\partial o(L_iv_1, R_iv_2)}{\partial (L_iv_i)} \right\| \left\| L_i\right\| \\
                &\le C_1C_2C_3,
            \end{align*} 
            where the last inequality follows by $\|\diag\{u^\prime\}\| = \|u^\prime\|_{\infty} \le C_3$ by condition (\ref{eq:vanishCond3}) together with conditions (\ref{eq:vanishCond1}) and (\ref{eq:vanishCond2}). By setting $C_0 = C_1C_2C_3$, the statement holds from condition (\ref{eq:vanishCond5}).
    	\end{proof}
    	
    	We have for all $1\le t \le T$, \begin{equation} \label{eq:b1}
    	    \begin{aligned}
                 \norm{\frac{\partial h_t}{\partial h_{t-1}}} &= \norm{ \sum_{k=1}^{N+1} \frac{\partial h_t}{\partial n_k^{\mathrm{leaf}}} \mathbbm{1}\left\{ n_k^{\mathrm{leaf}} = h_{t-1} \right\} } \le \norm{ \sum_{k=1}^{N+1} \frac{\partial h_t}{\partial n_k^{\mathrm{leaf}}}}\le \sum_{k=1}^{N+1} \norm{\frac{\partial h_t}{\partial n_k^{\mathrm{leaf}}}} = \sum_{k=1}^{N+1} \norm{\frac{\partial P_{l_k}^k}{\partial P_{0}^k}} \\
                 &= \sum_{k=1}^{N+1} \norm{\prod_{j=1}^{l_k} \frac{\partial P_j^k}{\partial P_{j-1}^k}} \le \sum_{k=1}^{N+1} \prod_{j=1}^{l_k} \norm{ \frac{\partial P_j^k}{\partial P_{j-1}^k}} \le \sum_{k=1}^{N+1} C_0^{l_k},
    	    \end{aligned}
    	\end{equation}
        where the last inequality follows by Lemma \ref{lemma:nodechild}.
        
        Note that $l_1, \ldots, l_{N+1}$ are the lengths of all the paths starting from the root node to the leaf node of a full binary tree. Lemma \ref{lemma:length} gives an upper-bound of the sum of exponents of these lengths.
    
    	\begin{lemma} \label{lemma:length}
    	    Suppose $l_k, 1\le k \le N+1$ are the lengths of the $N+1$ paths of the full binary tree $T_t^{\pred}$. Then for any $0<C_0<\frac{1}{2}$, there exists a constant $\epsilon = \epsilon(C_0), 0<\epsilon<1$, such that  \begin{equation*}
    	        \sum_{k=1}^{N+1} C_0^{l_k} \le 1 - \epsilon.
            \end{equation*}
        \end{lemma}
        
        \begin{proof}
            We prove by induction on $N$ that \begin{equation} \label{eq:induction}
                \sum_{k=1}^{N+1} C_0^{l_k} \le C_0^{N} + \sum_{k=1}^{N} C_0^k. 
            \end{equation}
            
            For $N=1$, the tree $T_t^\pred$ has exactly one internal node and two leaf nodes, thus there is only one tree structure for $T_t^\pred$ if we do not consider isomorphisms. We have $\{l_1, l_2, l_3\} = \{2, 2, 1\}$ and $\sum_{k=1}^{N+1} C_0^{l_k} = 2C_0^2+C_0 = C_0^{N} + \sum_{k=1}^{N} C_0^k $.
            
            Suppose that equation (\ref{eq:induction}) holds for $N\ge 1$, and we consider the case of $N+1$. For a tree $T$, recall the definition of $\calI(T)$ in Section \ref{sec:loss}. Given the full binary tree $T_t^\pred$, we denote $n_0$ to be the node that has the largest index in $\calI(T_t^\pred)$ among all $N+1$ internal nodes of $T_t^\pred$. It is obvious that both child nodes of $n_0$ are leaf nodes (if not, say the left child of $n_0$ is also an internal nodes, then it should have a larger index than $n_0$, which leads to a contradiction). We use $T_0$ to denote the tree obtained by removing the two child nodes of $n_0$ from the tree $T_t^\pred$ ($n_0$ is a leaf node of $T_0$). Then $T_0$ has exactly $N$ internal nodes. We use $l_1, \ldots, l_{N+2}$ and $l_1^\prime, \ldots, l_{N+1}^\prime$ to denote the lengths of all paths of $T_t^\pred$ and $T_0$, respectively. Without loss of generality, we denote the length of the path that ends at $n_0$ in $T_0$ as $l_{N+1}^\prime$, and the length of those two paths that pass through $n_0$ in $T_t^\pred$ as $l_{N+1}$ and $l_{N+2}$, respectively.
            
            Note that $l_i = l_i^\prime, 1\le i \le N$ and $l_{N+1} = l_{N+2} = l_{N+1}^\prime + 1$. We have \begin{align*}
                \sum_{k=1}^{N+2} C_0^{l_k} &= \sum_{k=1}^{N} C_0^{l_k^\prime} + 2C_0^{l_{N+1}^\prime+1}\\
                &= \sum_{k=1}^{N+1} C_0^{l_k^\prime} + (2C_0-1) C_0^{l_{N+1}^\prime} \\
                &\le C_0^{N} + \sum_{k=1}^{N} C_0^k + (2C_0-1) C_0^{N} \\
                &= C_0^{N+1} + \sum_{k=1}^{N+1} C_0^k,
            \end{align*}
            where the inequality follows from the induction equation $\sum_{k=1}^{N+1} C_0^{l_k^\prime} \le C_0^{N} + \sum_{k=1}^{N} C_0^k$ and the facts that $2C_0-1<0$ and $l_{N+1}^\prime \le N$. This complete the induction step.
            
            It remains to define the constant $\epsilon$. Since $C_0 < \frac{1}{2}$, we have \begin{align*}
                \sum_{k=1}^{N+1} C_0^{l_k} \le C_0^{N} + \sum_{k=1}^{N} C_0^k &< 2^{-N} + \sum_{k=2}^{N} 2^{-k} + C_0\\
                &= 2^{-N} + \sum_{k=1}^{N} 2^{-k} - \left(\frac{1}{2}-C_0\right) \\
                &= 1 - \left(\frac{1}{2}-C_0\right).
            \end{align*} 
            
            Taking $\epsilon =  \frac{1}{2}-C_0 > 0$ finishes the proof for Lemma \ref{lemma:length}.
        \end{proof}
        
        By (\ref{eq:b1}) and Lemma \ref{lemma:length}, for every $t, 1\le t\le T$ we have $$\norm{\frac{\partial h_t}{\partial h_{t-1}}} \le \eta \triangleq 1-\epsilon < 1.$$ 
        
        Combining with condition (\ref{eq:vanishCond4}), we have
        \begin{equation*}
            \norm{\frac{\partial \calE_T}{\partial h_T} \left( \prod_{1 < t \le T} \frac{\partial h_t}{\partial h_{t-1}} \right) \frac{\partial^+ h_{1}}{\partial \phi}} \le \norm{\frac{\partial \calE_T}{\partial h_T}}\prod_{1 < t \le T} \norm{\frac{\partial h_t}{\partial h_{t-1}} }\norm{\frac{\partial^+ h_{1}}{\partial \phi}} \le C_4 \eta^{T-1} \norm{\frac{\partial^+ h_{1}}{\partial \phi}}.
        \end{equation*}
        
        As $\eta<1$, we have $\norm{\frac{\partial \calE_T}{\partial h_T} \left( \prod_{1 < t \le T} \frac{\partial h_t}{\partial h_{t-1}} \right) \frac{\partial^+ h_{1}}{\partial \phi}}$ goes to 0 exponentially with $T \rightarrow \infty$.
        \qed
    
        \subsection{Disscussion of Theorem \ref{thm:vanishing}} \label{appendix:discussionvanishing}
        We next discuss the feasibility of conditions (\ref{eq:vanishCond1}) -- (\ref{eq:vanishCond5}). Condition (\ref{eq:vanishCond1}) requires all the weight matrices in $\calL \cup \calR$ to have the spectral norm no larger than $C_1$. For sigmoid, condition (\ref{eq:vanishCond2}) holds for $C_2 = \frac{1}{4}$ while for tanh and ReLU it holds for $C_2 = 1$. Condition (\ref{eq:vanishCond3}) bounds the spectral norm of the gradient of each binary function. If the binary operation $o$ is addition, then the Jacobian matrix $\frac{\partial o(L_i v_1, R_iv_2)}{\partial (L_i v_1)}$ is simply the identity matrix and its spectral norm equals to 1. For vector entry-wise multiplication, note that\begin{align*}
            \left\|\frac{\partial [(L_i v_1)\odot (R_iv_2)]}{\partial (L_i v_1)}\right\| = \left\|\diag\{R_i v_2\}\right\| = \left\|R_i v_2\right\|_\infty &\le \|R_i\|_\infty \|v_2\|_\infty \\
            &\le \sqrt{p}\|R_i\| \|v_2\|_\infty \le \sqrt{p}C_1C_5,
        \end{align*}
        
        where $C_5 \triangleq \max_{v\in \calV_t, 1\le t \le T} \|v\|_\infty$ is the upper bound of the infinity norm of all vectors on the predicted trees. Therefore, if $\odot \in \calO$, we should have $C_3 \ge \sqrt{p}C_1 C_5$. Condition (\ref{eq:vanishCond5}) holds when the scale of weight matrices or vectors on nodes of the predicted trees are small. In the experiments we have $C_2=1$ and $p=100$. Note also that $C_5=1$ if each $v$ is the outcome of sigmoid or tanh, or entry-wise product of such vectors; in presence of addition this no longer holds however computational experiments have established that $C_5\le 1$ even if addition is a candidate binary operation. Then condition (\ref{eq:vanishCond5}) holds for $C_1 \approx 0.223$ which has been observed in our experiments. 
        
        It is worth to mention that condition (\ref{eq:vanishCond4}) is mild since we only require the norm to be bounded by a sufficiently large constant. By (\ref{eq:loss}), we have \begin{equation*}
        	    \calE_t = \mathbb{E}_{(X,Y)}\big[\lambda_1 l(y_t,q_t)+\lambda_2 \min_{\overline{T}\in {\Iso}(T_t^{\target})} TD(\overline{T},T_t^{\pred}) + \lambda_3 \sum_{k=0}^{N-1} m(\calN_k^t)\big] + \frac{\lambda_4}{T}\sum_{\phi \in \Phi} \|\phi \|^2 \;.
    	\end{equation*}
    	
    	There are three terms in the expression of $\calE_t$ that are related to $h_t$, namely, the standard loss term, the tree distance term, and the scoring margin term. We bound them one by one. Since the number of samples is finite, we only consider each term for one data point in the following (i.e. we ignore the expectation).
    	
    	Assume that $ \norm{\frac{\partial l(y_t, q)}{\partial q}} \le C_6, \norm{\frac{\partial g(x_t, h; \Gamma)}{\partial h}} \le C_7 $ hold. Then the loss term is bounded by $$\norm{\frac{\partial l(y_t, q_t)}{\partial h_t}} = \norm{\frac{\partial l(y_t, q_t)}{\partial q_t}\frac{\partial q_t}{\partial h_t}} \le \norm{\frac{\partial l(y_t, q_t)}{\partial q_t}}\norm{\frac{\partial q_t}{\partial h_t}} = \norm{\frac{\partial l(y_t, q_t)}{\partial q_t}}\norm{\frac{\partial g(x_t, h_t; \Gamma)}{\partial h_t}} \le C_6C_7.$$ 
    	
    	If we use $h_t^\target$ to denote the vector of the root node of the ground truth tree $T_t^\target$ and assume that $\norm{h_t^\target} \le C_8$ for all $t$. Note that $h_t$ only appears at the root node of the predicted tree $T_t^\pred$, and the definition of TD can be regarded as a summation over many norms of vector differences. Then the only term in TD that includes $h_t$ is $\norm{h_t - h_t^\target}^2$. Therefore, the tree distance term is bounded by 
    	\begin{align*}\norm{\frac{\partial \mathrm{TD}(T_t^\target, T_t^\pred)}{\partial h_t}}& = \left\|\frac{\partial \norm{h_t - h_t^\target}^2}{\partial h_t}\right\| \\ & = 2 \norm{h_t - h_t^\target} \le 2\left( \norm{h_t}+\norm{h_t^\target} \right) \le 2(C_5+C_8).
    	\end{align*}
    	
    	Again, note that $h_t$ only appears in the term $$m(\calN_{N-1}^{t}) = \frac{1}{M} \min \Big\{ M, \alpha(h_t; \Theta) - \alpha(c_{N-1}^{**}; \Theta)\Big\}.$$ If we assume that $\norm{\frac{\partial \alpha(h; \Theta)}{\partial h}} \le C_9$ holds for any vector $h\in \mathbb{R}^p$, then the scoring margin term is bounded by $$ \norm{\frac{\partial \sum_{k=0}^{N-1}m(\calN_k^t)}{\partial h_t}} = \norm{\frac{\partial m(\calN_{N-1}^t)}{\partial h_t}} \le \norm{\frac{\partial \alpha(h_t; \Theta)}{\partial h_t}} \le C_9.$$
    	
    	In conclusion, if we assume the existence of constants $C_5, C_6, C_7, C_8$, and $C_9$, then the gradient of loss function $\calE_t$ with respect to the hidden state $h_t$ is bounded. Note that in practice $C_5$ and $C_8$ are about the norm of a finite set of vectors, and $C_6, C_7, C_9$ bound the norm of some simple functions or networks which in practice are all bounded. Therefore, we can easily argue that these constants do exist and thus the condition (\ref{eq:vanishCond4}) is mild.
    
        In summary, gradient vanishing frequently appears in practice.
         
        \subsection{Discussion of Theorem \ref{thm:explodingapp}} \label{appendix:necessary}
            In this section, we discuss the conditions appearing in the Theorem \ref{thm:explodingapp}. As the proof is similar to the proof of Theorem \ref{thm:vanishing} we omit it here.
            
            The conditions listed in Theorem \ref{thm:explodingapp} are common in practice since the quantity $(N+1)^{-\frac{1}{3l_{\min}}}$ is smaller than 1. If we have $\tanh \in \calU$ or $\text{ReLU} \in \calU$, then the first condition above is automatically achieved. Besides, if the addition operation belongs to $\calB$, then the third and the fourth conditions are both fulfilled. In summary, gradient exploding is frequent in practice.
        
    \section{Details of Experimental Study} \label{appendix:experiments}   
    \subsection{Implementation Details} \label{appendix:experimentsDetails}
        For our implementation of RRNN-GRU, we use PyTorch and train on Nvidia 1080 Ti GPUs or Intel Skylake CPUs. In order for our choices of cell structures to be differentiable with respect to the parameters of the scoring network, we evaluate softmax over the scores of all potential vectors at each node in the cell. Gradient clipping is used for RRNN-GRU, but random hyperparameter search often allows large gradient magnitudes. With the optimal hyperparameters, training of RRNN-GRU takes approximately ten hours for the Wikipedia dataset, one hour for PTB, and eight hours for SST, which is longer than the GRU training time since the RRNN-GRU weights must adapt to multiple placements within the cell structure. 
        %This acts as a form of regularization which allows RRNN to perform better than GRU on small datasets without requiring more parameters. 
        
        For the RRNN model, we use batch normalization to stabilize training.  
        The RRNN training time on the Wikipedia dataset with 5,000 samples is 40 hours on a CPU of a 12-core server. We find RRNN to be faster on a CPU than GPU due to its structure searching algorithm, but RRNN-GRU's algorithm runs faster on GPUs.
        
    \subsection{Hyperparameters} \label{appendix:hyperparameters}
        \begin{enumerate}
            \item GRU on Wiki-5k: batch size of 18, learning rate of $1.71\times10^{-3}$, and $\ell_2$-regularization coefficient of $3.60\times10^{-7}$.
            \item RRNN on Wiki-5k: batch size of 16, learning rate of $10^{-3}$, and scoring network hidden size of 256. $\lambda_1=1, \lambda_2=10^{-3}, \lambda_3=10^{-3}, \lambda_4=10^{-5}$. 
            \item GRU on Wiki-10k: batch size of 41, learning rate of $1.4231\times10^{-3}$, and $\ell_2$-regularization coefficient of $1.2124\times10^{-11}$.
            \item RRNN-GRU on Wiki-10k: batch size of 128, learning rate of $10^{-3}$, and scoring network hidden size of 64. Training alternates between the $L, R$, and $b$ weights and the scoring network every five epochs. $\lambda_1=1, \lambda_2=0.1, \lambda_3=10^{-8}, \lambda_4=0.003$. Gradients are clipped to the maximum norm of 1. 
            \item GRU on SST: learning rate of $4.85\times10^{-4}$, batch size of 3, and $\ell_2$ weight decay coefficient of $2.11\times10^{-12}$.
            \item RRNN-GRU on SST: learning rate of $1.06\times10^{-5}$, $\lambda_1=1, \lambda_2=1.76\times10^{-6}, \lambda_3=2.67\times10^{-12}, \lambda_4=5.47\times10^{-5}$, max-margin of $5.47\times10^{-5}$, scoring network hidden size of 10 nodes, gradients clipped to the norm of 46.3, alternating training every epoch.
            \item GRU on PTB: learning rate of $5.29\times 10^{-4}$, batch size of 6, and $\ell_2$ weight decay coefficient of $2.71\times 10^{-15}$.
            \item RRNN-GRU on PTB: batch size of 116, learning rate of $3.03\times10^{-4}$, $\lambda_1=1, \lambda_2=4.16\times10^{-3}, \lambda_3=1.22\times10^{-13}, \lambda_4=1.36\times10^{-3}$, max scoring margin of 1.74, maximum gradient magnitude of 1.64, scoring hidden size of 137, and alternating training every epoch.
        \end{enumerate}
            
            We next list the hyperparamter search ranges for RRNN-GRU in Table \ref{table:searchRRNNGRU} and GRU in Table \ref{table:searchGRU}.
            \begin{table}[!h]
                \caption{Hyperpameter search range for RRNN-GRU}
                \label{table:searchRRNNGRU}
                \centering
                \begin{tabular}{llll}
                    \toprule
                    & Wiki-10k & SST & PTB \\
                    \midrule
                    Batch size & $[1, 316]$& $[1, 316]$& $[1, 316]$ \\
                    Learning rate & $[10^{-5}, 10^{-2}]$& $[10^{-5}, 10^{-2}]$& $[10^{-5}, 10^{-2}]$ \\
                    $\lambda_2$ & $[10^{-2}, 1]$ & $[10^{-5}, 10^{-2}]$ & $[3\times10^{-2}, 3]$\\
                    $\lambda_3$ & $[10^{-16}, 1]$ & $[10^{-5}, 10^{-2}]$ & $[3\times10^{-16}, 3\times10^{-2}]$\\
                    $\lambda_4$ & $[10^{-6}, 10^{-2}]$ & $[10^{-8}, 10^{-4}]$ & $[3\times10^{-6}, 3\times10^{-2}]$\\
                    Scoring margin $M$ & $[0.1, 10]$ & $[0.1, 10]$ & $[0.1, 10]$ \\
                    Gradient clipping threshold & $[0.1, 100]$ & $[0.1, 100]$ & $[0.1, 100]$ \\
                    Alternate frequency & $[1,10]$ & $[1,10]$ & $[1,10]$ \\
                    \bottomrule
                \end{tabular}
            \end{table}

            \begin{table}[!h]
                \caption{Hyperpameter search range for GRU}
                \label{table:searchGRU}
                \centering
                \begin{tabular}{llll}
                    \toprule
                    & Wiki-5k/Wiki-10k & SST & PTB \\
                    \midrule
                    Batch size  & $[8, 256]$ & $[4, 128]$ & $[8, 256]$   \\
                    Learning rate  & $[10^{-5}, 10^{-1}]$ & $[10^{-6}, 10^{1}]$ & $[10^{-4}, 10^{-1}]$ \\
                    $\ell_2$ weight decay coefficient & $ [10^{-16}, 1]$ & $ [10^{-16}, 10^{-2}]$ & $ [10^{-16}, 1]$ \\
                    \bottomrule
                \end{tabular}
            \end{table}
            
\end{appendices}

\end{document}